\documentclass[12pt]{article}

\usepackage{mathtools}
\usepackage[mathscr]{euscript}

\usepackage{amsmath, amsthm}
\usepackage{natbib}
\usepackage{hyperref}
\usepackage{thmtools}
\usepackage{thm-restate}

\theoremstyle{plain}
\newtheorem{lemma}{Lemma}[section]
\newtheorem*{lemma*}{Lemma}
\newtheorem{corollary}[lemma]{Corollary}
\newtheorem*{corollary*}{Corollary}

\usepackage[algo2e, ruled]{algorithm2e}

\theoremstyle{definition}
\newtheorem{theorem}[lemma]{Theorem}
\newtheorem*{theorem*}{Theorem}
\newtheorem*{problem*}{Problem}

\usepackage{bm}

\usepackage{MnSymbol}
\DeclareMathAlphabet\mathbb{U}{msb}{m}{n}
\usepackage{xpatch}
\usepackage{fullpage}

\def\Rset{\mathbb{R}}

\let\Pr\undefined
\let\P\undefined
\DeclareMathOperator*{\Pr}{\mathbb{P}}
\DeclareMathOperator*{\P}{\mathbb{P}}
\DeclareMathOperator*{\E}{\mathbb E}

\DeclareMathOperator*{\argmin}{argmin}

\DeclareMathOperator*{\dom}{dom}
\DeclareMathOperator*{\cont}{cont}
\DeclareMathOperator*{\core}{core}
\DeclareMathOperator*{\cone}{cone}
\DeclareMathOperator*{\conv}{conv}
\DeclareMathOperator*{\diam}{diam}
\DeclareMathOperator*{\poly}{poly}

\DeclareMathOperator{\Reg}{\mathsf{Reg}}
\DeclareMathOperator{\SwapReg}{\mathsf{SwapReg}}
\DeclareMathOperator{\BSReg}{\mathsf{BayesSwapReg}}
\DeclareMathOperator{\PReg}{\mathsf{ProcrustesReg}}

\DeclarePairedDelimiter{\bracket}{[}{]}
\DeclarePairedDelimiter{\curl}{\{}{\}}
\DeclarePairedDelimiter{\norm}{\lVert}{\rVert}
\DeclarePairedDelimiter{\paren}{(}{)}
\DeclarePairedDelimiter{\tri}{\langle}{\rangle}

\newcommand{\cA}{\mathcal{A}}
\newcommand{\cC}{\mathcal{C}}
\newcommand{\cD}{\mathcal{D}}
\newcommand{\cF}{\mathcal{F}}

\newcommand{\cL}{\mathcal{L}}

\newcommand{\cX}{\mathcal{X}}
\newcommand{\cY}{\mathcal{Y}}

\newcommand{\sA}{{\mathscr A}}

\newcommand{\sC}{{\mathscr C}}

\newcommand{\sH}{{\mathscr H}}

\newcommand{\sK}{{\mathscr K}}
\newcommand{\sL}{{\mathscr L}}

\newcommand{\sP}{{\mathscr P}}

\newcommand{\sS}{{\mathscr S}}
\newcommand{\sT}{{\mathscr T}}

\newcommand{\sX}{{\mathscr X}}
\newcommand{\sY}{{\mathscr Y}}
\newcommand{\sZ}{{\mathscr Z}}

\newcommand{\bdc}{{\mathbf c}}

\newcommand{\bx}{{\mathbf x}}

\newcommand{\ODelta}{{\overline{\Delta}}}

\newcommand{\sfp}{{\mathsf p}}
\newcommand{\sfq}{{\mathsf q}}

\newcommand{\wt}{\widetilde}

\newcommand{\eps}{\varepsilon}
\newcommand{\ignore}[1]{}

\newcommand{\wtu}{\widetilde{u}}
\newcommand{\wttheta}{\widetilde{\theta}}
\newcommand{\wtcA}{\widetilde{\cA}}
\newcommand{\wtcF}{\widetilde{\cF}}
\newcommand{\wtR}{\widetilde{R}}

\newtheorem{question}{Open Question}

\title{Pseudonorm Approachability and \\
Applications to Regret Minimization}
\author{Christoph Dann  \\ Google Research \and Yishay Mansour \\ Google Research \and Mehryar Mohri \\ Google Research \and Jon Schneider \\ Google Research \and Balasubramanian Sivan \\ Google Research}

\usepackage[toc, page, header]{appendix}
\begin{document}

\maketitle

\begin{abstract}%

Blackwell's celebrated approachability theory provides a general
framework for a variety of learning problems, including regret
minimization. However, Blackwell's proof and implicit algorithm
measure approachability using the $\ell_2$ (Euclidean) distance. We
argue that in many applications such as regret minimization, it is
more useful to study approachability under other distance metrics,
most commonly the $\ell_\infty$-metric. But, the time and space
complexity of the algorithms designed for
$\ell_\infty$-approachability depend on the dimension of the space of
the vectorial payoffs, which is often prohibitively large. Thus, we
present a framework for converting high-dimensional
$\ell_\infty$-approachability problems to low-dimensional
\emph{pseudonorm} approachability problems, thereby resolving such
issues.  We first show that the $\ell_\infty$-distance between the
average payoff and the approachability set can be equivalently defined
as a \emph{pseudodistance} between a lower-dimensional average vector
payoff and a new convex set we define. Next, we develop an algorithmic
theory of pseudonorm approachability, analogous to previous work on
approachability for $\ell_2$ and other norms, showing that it can be
achieved via online linear optimization (OLO) over a convex set given
by the Fenchel dual of the unit pseudonorm ball. We then use that to
show, modulo mild normalization assumptions, that there exists an
$\ell_\infty$-approachability algorithm whose convergence is
independent of the dimension of the original vectorial payoff. We
further show that that algorithm admits a polynomial-time complexity,
assuming that the original $\ell_\infty$-distance can be computed
efficiently. We also give an $\ell_\infty$-approachability algorithm
whose convergence is logarithmic in that dimension using an FTRL
algorithm with a maximum-entropy regularizer. Finally, we illustrate
the benefits of our framework by applying it to several problems in
regret minimization.

\end{abstract}


\section{Introduction}

The notion of approachability introduced by \cite{Blackwell1956} can
be viewed as an extension of von Neumann's minimax theorem
\citep{vonNeumann1928} to the case of vectorial
payoffs. \citeauthor{Blackwell1956} gave a simple example showing that
the straightforward analog of von Neumann's minimax theorem does not
hold for vectorial payoffs. However, in contrast with this negative
result for one-shot games, he proved that, in repeated games, a player
admits an adaptive strategy guaranteeing that their average payoff
\emph{approaches} a closed convex set in the limit, provided that the
set satisfies a natural separability condition.

The theory of Blackwell approachability is intimately connected with
the field of online learning for the reason that the problem of
\textit{regret minimization} can be viewed as an approachability
problem: in particular, the learner would like their vector of regrets
(with respect to each competing benchmark) to converge to a
non-positive vector.  In this vein, \cite{AbernethyBartlettHazan2011}
demonstrated how to use algorithms for approachability to solve a
general class of regret minimization problems (and conversely, how to
use regret minimization to construct approachability
algorithms). However, applying their reduction sometimes leads to
suboptimal regret guarantees -- for example, for the specific case of
minimizing external regret over $T$ rounds with $K$ actions, their
reduction results in an algorithm with $O(\sqrt{TK})$ regret (instead
of the optimal $O(\sqrt{T\log K})$ regret bound achievable by
e.g. multiplicative weights).

One reason for this suboptimality is the choice of distance used to
define approachability. Both \citeauthor{Blackwell1956} and
\citeauthor{AbernethyBartlettHazan2011}\! consider approachability
algorithms that minimize the Euclidean ($\ell_2$) distance between
their average payoff and the desired set. We argue that, for
applications to regret minimization, it is often more useful to study
approachability under other distance metrics, most commonly
approachability under the \textit{$\ell_{\infty}$ metric} to the
\textit{non-positive orthant}, which is well suited to capture the
fact that regret is a \textit{maximum} over various competing
benchmarks. This has been observed in the past in several recent
publications \citep{Perchet2015,Shimkin2016,Kwon2021}.  In particular,
by constructing algorithms for $\ell_{\infty}$ approachability, it is
possible to naturally recover a $O(\sqrt{T\log K})$ external regret
learning algorithm (and algorithms with optimal regret guarantees for
many other problems of interest).

However, there is still one significant problem with developing regret
minimization algorithms via $\ell_{\infty}$ approachability (or any of
the forms of approachability previously mentioned): the time and space
complexity of these algorithms depends polynomially on the dimension
$d$ of the space of vectorial payoffs, which in turn equals the number
of benchmarks we compete against in our regret minimization
problem. In some regret minimization settings, this can be
prohibitively expensive. For example, in the setting of swap regret
(where the benchmarks are parameterized by the $d = K^K$ swap
functions mapping $[K]$ to $[K]$), this results in algorithms with
complexity exponential in $K$. On the other hand, there exist
algorithms, e.g.\ \citep{BlumMansour2007}, which are both efficient
($\poly(K)$ time and space) and obtain optimal regret guarantees.

\subsection{Main results}

In this paper, we present a framework for converting
\textit{high-dimensional} $\ell_{\infty}$ approachability problems to
\textit{low-dimensional} ``pseudonorm'' approachability problems, in
turn resolving many of these issues. To be precise, recall that the
setting of approachability can be thought of as a $T$-round repeated
game, where in round $t$ the learner chooses an action $p_t$ from some
convex ``action'' set $\sP \subseteq \Rset^n$, the adversary
simultaneously chooses an action $\ell_t$ from some convex ``loss''
set $\sL \subseteq \Rset^m$, and the learner receives a vector-valued
payoff $u(p_t, \ell_t)$, where $u\colon \sP \times \sL \rightarrow
\Rset^{d}$ is a $d$-dimensional bilinear function. The learner would
like the $\ell_{\infty}$ distance between their average payoff
$\frac{1}{T}\sum u(p_t, \ell_t)$ and some convex set\footnote{For
simplicity, throughout this paper we assume $\sS$ to be the negative
orthant $(-\infty, 0]^d$, as this is the case most relevant to regret
minimization. However, much of our results extend straightforwardly
to arbitrary convex sets.} $\sS$ to be as small as possible.

We first demonstrate how to construct a new $d'$-dimensional bilinear
function $\wtu\colon \sP \times \sL \rightarrow \Rset^{d'}$, a new
convex set $\sS' \subseteq \Rset^{d'}$, and a pseudonorm\footnote{In
this paper a \textit{pseudonorm} is a function $f\colon \Rset^{d'}
\rightarrow \Rset_{\geq 0}$ which satisfies most of the properties of
a norm (e.g. positive homogeneity, triangle inequality), but may be
asymmetric (there may exist $x$ where $f(x) \neq f(-x)$) and may not
be definite (there may exist $x \neq 0$ where $f(x) = 0$). Just as a
norm $\norm{\cdot}$ defines a distance between $x, y \in \Rset^{d'}$
via $\norm{x - y}$, a pseudonorm $f$ defines the pseudodistance $f(x -
y)$.} $f$ such that the ``pseudodistance'' between the average
modified payoff $\frac{1}{T}\sum \wtu(p_t, \ell_t)$ and $\sS'$ is
equal to the $\ell_{\infty}$ distance between the original average
payoff and $\sS$. Importantly, the new dimension $d'$ is equal to $mn$
and is independent of the original dimension $d$.

We then develop an algorithmic theory of pseudonorm approachability
analogous to that developed in \citep{AbernethyBartlettHazan2011} for
the $\ell_2$ norm and \citep{Shimkin2016, Kwon2021} for other norms,
showing that, in order to perform pseudonorm approachability, it
suffices to be able to perform online linear optimization (OLO) over a
convex set given by the Fenchel dual of the unit pseudonorm ball (and
that the rate of approachability is directly related to the regret
guarantees of this OLO subalgorithm). This has the following
consequences for approachability:

\begin{itemize}
    \item First, by solving this OLO problem with a quadratic
      regularized Follow-The-Regularized-Leader (FTRL) algorithm, we
      show (modulo mild normalization assumptions on the sizes of
      $\sP$, $\sL$, and $u$) that there exists a pseudonorm
      approachability algorithm (and hence an $\ell_{\infty}$
      approachability algorithm for the original problem) which
      converges at a rate of $O(nm/\sqrt{T})$. We additionally provide
      a stronger bound on the rate which scales as
      $O(D_{u}D_{p}D_{\ell}/\sqrt{T})$ where $D_{p} = \diam \sP$,
      $D_{\ell} = \diam \sL$, and $D_{u}$ is the maximum $\ell_2$ norm
      of the set of vectors formed by taking the coefficients of
      components of $u$ (Theorem \ref{thm:poly_regret}). In
      comparison, the best-known generic guarantee for
      $\ell_{\infty}$-approachability prior to this work converged at
      a $d$-dependent rate of $O(\sqrt{(\log d)/T})$.
      
    \item \sloppy{Second, we show that as long as we can evaluate the
      original $\ell_{\infty}$-distance between $\frac{1}{T}\sum
      u(p_t, \ell_t)$ and $(-\infty, 0]^d$ efficiently, we can
      implement the above algorithm in $\poly(m, n)$ time per round
      (Theorem \ref{thm:efficient_alg}). This has the following
      natural consequence for the class of regret minimization
      problems that can be written as $\ell_{\infty}$-approachability
      problems: \textit{if it is possible to efficiently compute some
        notion of regret for a sequence of losses and actions, then
        there is an efficient (in the dimensions of the actions and
        losses) learner that minimizes this regret}.}

      \item Finally, in some cases, the $O(\sqrt{(\log d)/T})$
        approachability rate from (inefficient)
        $\ell_{\infty}$-approachability outperforms the rate obtained
        by the quadratic regularized FTRL algorithm. We define a new
        regularizer whose value is given by finding the maximum
        entropy distribution of a subset of distributions of support
        $d$, and show that, by using this regularizer, we recover this
        $O(\sqrt{(\log d)/T})$ rate. In particular, whenever we can
        efficiently compute this maxent regularizer, there is an
        efficient learning algorithm with a $O(\sqrt{(\log d)/T})$
        approachability rate.
\end{itemize}

\noindent
We then apply our framework to various problems in regret minimization:

\begin{itemize}
    \item We show that our framework straightforwardly recovers a
      regret-optimal and efficient algorithm for swap regret
      minimization. Doing so requires computing the above maximum
      entropy regularizer for this specific case, where we show that
      it has a nice closed form. In particular, to our knowledge. this
      is the first approachability-based algorithm for swap regret
      that both is efficient and has the optimal minimax regret.
    
    \item In Section~\ref{sec:bayes_correlated}, we apply our
      framework to develop the first efficient contextual learning
      algorithms with low \textit{Bayesian swap regret}. Such
      algorithms have the property that if learners employ them in a
      repeated Bayesian game, the time-average of their strategies
      will converge to a Bayesian correlated equilibrium, a
      well-studied equilibrium notion in game theory (see
      e.g. \cite{bergemann2016bayes}).
    
    This notion of Bayesian swap regret was recently introduced by
    \cite{mansour2022}, who also provided an algorithm with low
    Bayesian swap regret, albeit one that is not computationally
    efficient. By applying our framework, we easily obtain an
    efficient contextual learning algorithm with $O(CK\sqrt{T})$
    Bayesian swap regret, resolving an open question of
    \cite{mansour2022} (here $C$ is the number of ``contexts'' /
    ``types'' of the learner).

    \item In Section~\ref{sec:rl}, we further analyze the application
      of our general $\ell_\infty$-approachability theory and
      algorithm to the analysis of reinforcement learning (RL) with
      vectorial losses. We point out how our framework can provide a
      general solution in the full information setting with known
      transition probabilities and how we can recover the best known
      solution for standard regret minimization in episodic RL. More
      importantly, we show how our framework and algorithm can lead to
      an algorithm for constrained MDPs with a significantly more
      favorable regret guarantee, logarithmic in the number of
      constraints $k$, in contrast with the $\sqrt{k}$-dependency of
      the results of \citet{MiryoosefiBrantleyDaumeDudikSchapire2019}.

\end{itemize}


\subsection{Related Work}

There is a wide literature dealing with various aspects of Blackwell's
approachability, including its applications to game theory,
regret minimization, reinforcement learning, and multiple extensions.

\citet{HartMasColell2000} described an adaptive procedure for players
in a game based on Blackwell's approachability, which guarantees that
the empirical distributions of the plays converges to the set of a
correlated equilibrium. This procedure is related to internal regret
minimization, for which, as shown by \citet{FosterVohra1999}, the
existence of an algorithm follows the proof of
\citet{HartMasColell2000}.  \citet{HartMasColell2001} further gave a
general class of adaptive strategies based on approachability.
Approachability has been widely used for calibration \citep{Dawid82},
see \cite{FosterHart2017} for a recent work on the topic.
Approachability and partial monitoring were studied in a series of
publications by \citet{Perchet2010,MannorPerchetStoltz2014,
  MannorPerchetStoltz2014Bis,
  PerchetQuincampoix2015,PerchetQuincampoix2018,KwonPerchet2017}. More
recently, approachability has also been used in the analysis of
fairness in machine learning \citep{ChzhenGiraudStoltz2021}.

Approachability has also been extensively used in the context of
reinforcement learning. \citet{MannorShimkin2003} discussed an
extension of regret minimization in competitive Markov decision
processes (MDPs) whose analysis is based on Blackwell's
approachability theory. \citet{MannorShimkin2004} presented a
geometric approach to multiple-criteria reinforcement learning
formulated as approachability conditions.
\citet{KalathiBorkarJain2014} presented strategies for approachability
for MDPs and Stackelberg stochastic games based on
\citeauthor{Blackwell1956}'s approachability theory.  More recently,
\citet{MiryoosefiBrantleyDaumeDudikSchapire2019} used approachability
to derive solutions for reinforcement learning with convex
constraints. 

The notion of approachability was further extended in several studies.
\citet{Vieille1992} used differential games with a fixed duration to
study weak approachability in finite dimensional
spaces. \citet{Spinat2002} formulated a necessary and sufficient
condition for approachability of non-necessary convex sets.
\citet{Lehrer2003Bis} extended \citeauthor{Blackwell1956}'s
approachability theory to infinite-dimensional spaces.

The most closely related work to this paper, which we build upon,
is that of \citet{AbernethyBartlettHazan2011} who showed that,
remarkably, any algorithm for Blackwell's approachability could be
converted into one for online convex optimization and vice-versa.
\citet{BernsteinShimkin2015} also discussed a related response-based
approachability algorithm.

\citet{Perchet2015} presented a specific study of $\ell_\infty$
approachability, for which they gave an exponential weight algorithm.
\citet[][Section~5]{Shimkin2016} studied approachability for an
arbitrary norm and gave a general duality result for an arbitrary norm
using Sion's minmax theorem. The pseudonorm duality theorem we prove,
using Fenchel duality, can be viewed as a generalization.
\citet{Kwon2016,Kwon2021} also presented a duality theorem similar to
that of \citet{Shimkin2016} which they used to derive a FTRL algorithm
for general norm approachability. They further treated the special
case of internal and swap regret. However, unlike the algorithms
derived in this work, the computational complexity of their swap
regret algorithm is in $O(K^K)$. This is also true of the paper of
\citet{Perchet2015} which also analyzes the swap regret problem.


It is known that if all players follow a swap regret minimization
algorithm, then the empirical distribution of their play converges to
a correlated equilibrium \citep{BlumMansour2007}.
\cite{HazanKale2008} showed a result generalizing this property to the
case of $\Phi$-regret and $\Phi$-equilibrium, where the $\Phi$-regret
is the difference between the cumulative expected loss suffered by the
learner and that of the best $\Phi$-modification of the sequence in
hindsight.
\citet{GordonGreenwaldMarks2008} further generalized the results of
\cite{HazanKale2008} to a more general class of $\Phi$-modification
regrets.  The algorithms discussed in \citep{GordonGreenwaldMarks2008}
are distinct from those discussed in this paper (they do not clearly
extend to the general approachability setting, and they require
significantly different computational assumptions than
ours). Nevertheless, they bare some similarity with our work.

\ignore{
\textbf{TODO: add phi regret papers (Hazan / Kale), (Gordon et al.)}
}

\section{Preliminaries}

\paragraph{Notation.} We use $[n]$ as a shorthand for
set $\{1, 2, \dots, n\}$. We write $\Delta_{d} = \{x \in \Rset^d \mid
x_i \geq 0, \sum x_i = 1 \}$ to denote the simplex over $d$ dimensions
and $\ODelta_{d} = \{x \in \Rset^d \mid x_i \geq 0, \sum x_i \leq 1
\}$ to denote the convex hull of the $d$-simplex with the origin.
$\conv(S)$ denotes the convex hull of the points in $S$, and
$\cone(S) = \{\alpha x \mid \alpha \geq 0, x \in \conv(S)\}$ the
convex cone generated by the points in $S$.

Some of the more standard proofs have been deferred to Appendix
\ref{app:omitted}.

\subsection{Blackwell approachability and regret minimization}
\label{sec:online_learning}

We begin by illustrating the theory of Blackwell approachability for
the specific case of the $\ell_{\infty}$-distance; this case is both
particularly suited to the application of regret minimization, and
will play an important role in the results (e.g. reductions to
pseudonorm approachability) that follow.

We consider a repeated game setting, where every round $t$ a learner
chooses an action $p_t$ belonging to a bounded\footnote{We bound the
entries of $\sP$, $\sL$, and $u$ within $[-1, 1]$ for convenience, but
it is generally easy to translate between different boundedness
assumptions (since almost all relevant quantities are linear). We
express the majority of our theorem statements (with the notable
exception of Theorem \ref{thm:poly_regret}) in a way that is
independent of the choice of bounds.} convex set $\sP \subseteq [-1,
  1]^{n}$, and an adversary simultaneously chooses a loss $\ell_t$
belonging to a bounded convex set $\sL \subseteq [-1, 1]^{m}$. Let
$u\colon \sP \times \sL \rightarrow [-1, 1]^d$ be a bounded
bilinear\footnote{We briefly note that all our results also hold for
\textit{biaffine} functions $u_i$; in particular, extending the loss
and action sets slightly (by replacing $\sP$ and $\sL$ with $\sP
\times \{1\}$ and $\sL \times \{1\}$) allows us to write any biaffine
function over the original sets as a bilinear function over the
extended sets.} vector-valued \textit{payoff function}, and let $\sS
\subseteq \Rset^d$ be a closed convex set with the property that for
every $\ell \in \sL$, there exists a $p \in \sP$ such that $u(p, \ell)
\in \sS$ (we say that such a set is ``separable''). When $d = 1$, the
minimax theorem implies that there exists a single $p \in \sP$ such
that for all $\ell$, $u(p, \ell) \in \sS$.

This is not true for $d > 1$, but the theory of Blackwell
approachability provides the following algorithmic analogue of this
statement. Define a \textit{learning algorithm} $\cA$ to be a
collection of functions $\cA_t \colon \sL^{t-1} \rightarrow \sP$ for
each $t \in [T]$, where $\cA_t$ describes how the learner decides
their action $p_t$ as a function of the observed losses $\ell_1,
\ell_2, \dots, \ell_{t-1}$ up until time $t-1$. Blackwell
approachability guarantees that there exists a learning algorithm
$\cA$ such that when $\cA$ is run on any loss sequence $\bm{\ell}$ the
resulting action sequence $\bm{p}$ has the property that:

\begin{equation}\label{eq:approachability}
  \lim_{T\rightarrow \infty} d_{\infty}
  \left(\frac{1}{T}\sum_{t=1}^{T}u(p_{t}, \ell_t), \sS\right) = 0.
\end{equation}

\noindent
(Here for $v, w \in \Rset^d$, $d_{\infty}(v, w) = \max_{i} |v_i -
w_i|$ represents the $\ell_{\infty}$ distance between $v$ and $w$).

As mentioned, one of the main motivations for studying Blackwell
approachability is its connections to regret minimization. In
particular, for a fixed choice of $u$, define

\begin{equation}\label{eq:regret_def}
  \Reg(\mathbf{p}, \bm{\ell}) = \max\left(\max_{i \in [d]}
  \left(\sum_{t=1}^{T} u_i(p_t, \ell_t)\right), 0\right).
\end{equation}

Note first that this definition of ``regret'' is exactly $T$ times the
$\ell_{\infty}$ approachability distance in the case where $\sS =
(\infty, 0]^d$ is the negative orthant; that is,

\begin{equation}\label{eq:regret}
  \Reg(\bm{p}, \bm{\ell}) = T \cdot d_{\infty}
  \left(\frac{1}{T}\sum_{t=1}^{T}u(p_{t}, \ell_t), (-\infty, 0]^d \right).
\end{equation}

But secondly, note that by choosing $\sP$, $\sL$, and $u$ carefully,
the definition $\eqref{eq:regret_def}$ can capture a wide variety of
forms of regret studied in regret minimization. For example:

\begin{itemize}
    \item When $\sP = \Delta_{K}$, $\sL = [0, 1]^K$, $d = K$, and
      $u(p, \ell)_i = \langle p, \ell \rangle - \ell_i$, $\Reg(\bm{p},
      \bm{\ell})$ is the \textit{external regret} of playing action
      sequence $\bm{p}$ against loss sequence $\bm{\ell}$; i.e., it
      measures the regret compared to the best single action.

    \item When $\sP = \Delta_{K}$, $\sL = [0, 1]^K$, $d = K^{K}$, and
      (for each function $\pi \colon [K]\rightarrow[K]$, $u(p, \ell)_\pi =
      \sum_{i=1}^{K}p_{i}(\ell_{i} - \ell_{\pi(i)})$, $\Reg(\bm{p},
      \bm{\ell})$ is the \textit{swap regret} of playing action
      sequence $\bm{p}$ against loss sequence $\bm{\ell}$; i.e., it
      measures the regret compared to the best action sequence
      $\bm{p'}$ obtained by applying a fixed swap function to sequence
      $\bm{p}$.

    \item When $\sP \subseteq \Rset^{n}$ is a convex polytope, $\sL =
      [0, 1]^n$, $d = |V(\sP)|$ (where $V(\sP)$ is the set of vertices
      of $\sP$), and (for each vertex $v \in V(\sP)$), $u(p, \ell)_{v}
      = \langle p, \ell \rangle - \langle v, \ell\rangle$, this
      captures the (external) regret from performing online linear
      optimization over the polytope $\sP$ (see Section
      \ref{sec:olo}).

    \item Finally, to illustrate the power of this framework, we
      present an unusual swap regret minimization that we call
      ``Procrustean swap regret minimization'' (after the orthogonal
      Procrustes problem, see \citep{gower2004procrustes}). Let $\sP =
      \{x \in \Rset^{n} \, \colon ||x||_2 \leq 1\}$ be the unit ball in
      $n$ dimensions, $\sL = [-1, 1]^{n}$, and, for each orthogonal
      matrix\footnote{Technically, this leads to an infinite
      dimensional $u$ (since the group of orthogonal matrices is
      infinite), but one can instead take an arbitrarily fine discrete
      approximation of the set. Indeed, one of the advantages of the
      results we present is that they are largely independent of the
      dimension $d$ of $u$.} $Q \in O(n)$, let $u(p, \ell)_{Q} =
      \langle p, \ell \rangle - \langle Qp, \ell \rangle$.
\end{itemize}

When the negative orthant $(-\infty, 0]^d$ is separable with respect
  to $u$, which is true in all of the above examples, the theory of
  Blackwell approachability immediately guarantees the existence of a
  sublinear regret learning algorithm $\cA$ for the corresponding
  notion of regret. Specifically, define the regret of a learning
  algorithm $\cA$ to be the worst-case regret over all possible loss
  vectors $\bm{\ell} \in \sL^{T}$; i.e.,

\begin{equation}\label{eq:regret_def2}
    \Reg(\cA) = \max_{\bm{\ell} \in \sL^{T}}\, \Reg(\bm{p}, \bm{\ell}),
\end{equation}

\noindent
where $p_t = \cA_{t}(\ell_1, \dots, \ell_{t-1})$. Then
\eqref{eq:approachability} implies that the same algorithm $\cA$
satisfies $\Reg(\cA) = o(T)$. Motivated by this application, we will
restrict our attention for the remainder of this paper to the setting
where $\sS = (-\infty, 0]^d$ and will assume (unless otherwise
  specified) that this $\sS$ is separable with respect to the bilinear
  function $u$ we consider.

In fact, the theory of $\ell_{\infty}$-approachability is constructive
and allows us to write down explicit algorithms $\cA$ along with
explicit (and in many cases, near optimal) regret bounds. However,
before we introduce these algorithms, we will need to introduce the
problem of online linear optimization.

\subsection{Online linear optimization}
\label{sec:olo}

Here we discuss algorithms for \textit{online linear optimization}
(OLO), a special case of online convex optimization where all the loss
functions are linear functions of the learner's action. These will
form an important primitive of our algorithms for approachability.

Let $\sX, \sY$ be two bounded convex subsets of $\Rset^r$ and consider
the following learning problem. Every round $t$ (for $T$ rounds) our
learning algorithm must choose an element $x_t \in \sX$ as a function
of $y_1, y_2, \dots, y_{t-1}$. The adversary simultaneously chooses a
loss ``function'' $y_t \in \sY$. This causes the learner to incur a
loss of $\langle x_t, y_t \rangle$ in round $t$. The goal of the
learner is to minimize their regret, defined as the difference between
their total loss and the loss they would have incurred by playing the
best fixed action in hindsight, i.e.,

$$\Reg(\bm{x}, \bm{y}) = \max_{x^* \in \sX}\left(\sum_{t=1}^{T} \langle x_t, y_t\rangle - \sum_{t=1}^{T} \langle x^*, y_t\rangle\right).$$

There are many similarities between this problem and the
approachability and regret minimization problems discussed in Section
\ref{sec:online_learning}. For example, if we take $\sX = \sP =
\Delta_{K}$ and $\sY = \sL = [0, 1]^{K}$, then OLO is equivalent to
the problem of external regret minimization. However, not all regret
minimization problems can be written directly as an instance of OLO --
for example, there is no clear way to write swap regret minimization
as an OLO instance. Eventually we will demonstrate how to apply OLO as
a subroutine to solve any regret minimization problem, but this will
involve a reduction to Blackwell approachability and will require
running OLO on different spaces than the action/loss sets $\sP$ and
$\sL$ directly (which is why we distinguish the action/loss sets for
OLO as $\sX$ and $\sY$ respectively).

There is an important subclass of algorithms for OLO known as
\textit{Follow the Regularized Leader (FTRL)} algorithms
\citep{ShalevShwartz2007,AbernethyHazanRakhlin2008}. An FTRL algorithm
is completely specified by a strongly convex function $R\colon \sX
\rightarrow \Rset$, and plays the action

\[
x_t = \argmin_{x \in \sX}\left(R(x) + \sum_{s=1}^{t-1}\langle x,
y_{s}\rangle \right).
\]

In words, this algorithm plays the action that minimizes the total
loss on the rounds until the present (``following the leader''),
subject to an additional regularization term. It is possible to
characterize the worst-case regret of an FTRL algorithm in terms of
properties of the regularizer $R$ and the sets $\sX$ and $\sY$ (see
e.g. Theorem 15 of \cite{hazan2016introduction}). For our purposes, we
will only need the following two results for specific regularizers.

\begin{lemma}[Quadratic regularizer, \citep{Zinkevich2003}]
\label{lem:quadratic_regularizer}
Let $D_x = \diam \sX$ and $D_y = \max_{y \in \sY} ||y||_2$. Let $x_0$
be an arbitrary element of $\sX$, and let $R(x) = ||x - x_0||^2$. Then,
the FTRL algorithm with regularizer $R$ incurs worst-case regret at
most $O(D_xD_y\sqrt{T})$.
\end{lemma}

\begin{lemma}[Negative entropy regularizer, \citep{KivinenWarmuth1995}]
\label{lem:negent_regularizer}
Let $\sX = \Delta_{d}$, $\sY = [0, 1]^d$, and let $R(x) =
\sum_{i=1}^{d} x_i\log x_i$ (where we extend this to the boundary of
$\sX$ by letting $0\log 0 = 0$). Then, the FTRL algorithm with
regularizer $R$ incurs worst-case regret at most $O(\sqrt{T \log d})$.
\end{lemma}

\subsection{Algorithms for $\ell_{\infty}$-approachability}
\label{sec:approachability}

We can now write down an explicit description of our algorithm for
$\ell_{\infty}$-approachability (in terms of a blackbox OLO algorithm)
and get quantitative bounds on the rate of convergence in the LHS of
\eqref{eq:approachability}. Let $\cF$ be an OLO algorithm for the sets
$\sX = \ODelta_{d}$ and $\sY = [-1, 1]^d$. Then we can describe our
algorithm $\cA$ for $\ell_{\infty}$-approachability as
Algorithm~\ref{algo:algorithm-linf-approach}.

\begin{algorithm2e}[ht]
\SetAlgoLined
Initialize $\theta_1$ to an arbitrary point in $\ODelta_d$\;
\BlankLine
\For{$t\leftarrow 1$  \KwTo $T$}{
  Choose $p_t \in \sP$ so that for all $\ell \in \sL$, $\langle
  \theta_t, u(p_t, \ell)\rangle \leq 0$. Such a $p_t$ is guaranteed to
  exist by the separability condition on $u$, since $\langle \theta, s
  \rangle \leq 0$ for any $\theta \in \ODelta_d$ and $s \in (-\infty,
  0]^d$\;
Play action $p_t$ and receive as feedback $\ell_t \in \sL$\;
Set $y_t = -u(p_t, \ell_t)$\;
Set $\theta_{t+1} = \cF(y_1, y_2, \dots, y_t)$.
}
\caption{Description of Algorithm $\cA$ for $\ell_{\infty}$-approachability}
\label{algo:algorithm-linf-approach}
\end{algorithm2e}

It turns out that we can relate the $\ell_{\infty}$ approachability
distance to $(-\infty, 0]^{d}$ (and hence the regret of this algorithm
  $\cA$) to the regret of our OLO algorithm $\cF$.

\begin{theorem}\label{thm:linf_approachability}
We have that
$$\Reg(\cA) = T \cdot
d_{\infty}\left(\frac{1}{T}\sum_{t=1}^{T}u(p_{t}, \ell_t), (-\infty,
0]^d \right) \leq \Reg(\cF).$$
\end{theorem}

If we let $\cF$ be the negative entropy FTRL algorithm (Lemma
\ref{lem:negent_regularizer})\footnote{There is a slight technical
difference between the sets $\cX = \ODelta_d$ and $\cY = [-1, 1]^d$ in
Algorithm \ref{algo:algorithm-linf-approach} and the sets $\cX' =
\Delta_d$ and $\cY' = [0, 1]^d$ in Lemma
\ref{lem:negent_regularizer}. However, note that if we map $x \in \cX$
to $(x_1, x_2, \dots, x_n, 1 - \sum x_i) \in \Delta_{d+1}$ and $y \in
[-1, 1]^d$ to $((1+y_1)/2, (1+y_2)/2, \dots, (1+y_d)/2, 0) \in [0,
  1]^{d+1}$, we preserve the inner product of $x$ and $y$ up to an
additive constant, which disappears when computing regret, and a
factor of $1/2$.}, we obtain the following regret guarantee for $\cA$.

\begin{corollary}
  \label{cor:linf_regret_minimization}
For any bilinear regret function $u\colon \sP \times \sL \rightarrow
[-1, 1]^d$, there exists a regret minimization algorithm $\cA$ with
worst-case regret $\Reg(\cA) = O(\sqrt{T \log d})$.
\end{corollary}

Equivalently, Corollary \ref{cor:linf_regret_minimization} can be
interpreted as saying that there exists an
$\ell_{\infty}$-approachability algorithm which approaches the
negative orthant at an average rate of $O(\sqrt{(\log d)/T})$. In
general, \eqref{eq:regret} lets us straightforwardly convert between
results for approachability and results for regret
minimization. Throughout the remainder of the paper we will primarily
phrase our results in terms of $\Reg(\cA)$, but switch between
quantities of interest when convenient.

\section{Main Results}
\label{sec:main}

Already Corollary \ref{cor:linf_regret_minimization} leads to a number
of impressive consequences. For example, when applied to the problem
of swap-regret minimization (where $u$ has dimension $d = K^K$), it
leads to a learning algorithm with $O(\sqrt{KT\log K})$ regret,
matching the best known regret bounds for this problem
\citep{BlumMansour2007}. However, the algorithm we obtain in this way
has two unfortunate properties.

First, since $u$ is $d$-dimensional, implementing the algorithm as
written above requires $O(\poly(d))$ time and space complexity (even
storing any specific $y_t$ or $\theta_t$ requires $\Omega(d)$
space). This is fine if $d$ is small, but in many of our applications
$d$ is much (e.g., exponentially) larger than the dimensions $m$ and
$n$ of the loss and action sets. For example, for swap regret we have
$d = K^K$ but $m = n = K$. Although Corollary
\ref{cor:linf_regret_minimization} gives us an optimal $O(\sqrt{KT\log
  K})$ swap regret algorithm, it takes exponential time / space to
implement (in contrast to other known swap regret algorithms, such as
that of \citet{BlumMansour2007}).

Secondly, although Corollary \ref{cor:linf_regret_minimization} has
only a logarithmic dependence on $d$, sometimes even this may be too
large (for example when we want to compete against an uncountable set
of benchmarks). In such cases, we would ideally like a regret bound
that depends on the action and loss sets but not directly on $d$.

In the following subsections, we will demonstrate a framework for
regret minimization that allows us to achieve both of these goals
(under some fairly light computational assumptions on $u$).

\subsection{Approachability for pseudonorms}
\label{sec:pseudonorm_approachability}

In Section \ref{sec:approachability}, we described the theory of
Blackwell approachability for a distance metric defined by the
$\ell_{\infty}$ norm (i.e., $||z||_{\infty} = \max_i |z_i|$). We begin
here by describing a generalization of this approachability theory to
functions we refer to as \emph{pseudonorms} and
\emph{pseudodistances}. A function $f \colon \Rset^{d} \to \Rset_{\geq
  0}$ is a pseudonorm if $f(0) = 0$, $f$ is positive homogeneous (for
all $z \in \Rset^{d}$, $\alpha \in \Rset_{\geq 0}$, $f(\alpha z) =
\alpha f(z)$), and $f$ satisfies the triangle inequality ($f(z + z')
\leq f(z) + f(z')$ for all $z, z' \in \Rset^{d}$). Note that unlike
norms, pseudonorms may not satisfy definiteness and are not
necessarily symmetric; it may be the case that $f(z) \neq
f(-z)$. However, all norms are pseudonorms. Note also that by the
positive homogeneity and the triangle inequality, a pseudonorm is a
convex function.  A pseudonorm $f$ defines a \emph{pseudodistance
function} $d_f$ via: $\forall z, z' \in \Rset^{d}, d_f(z, z') = f(z -
z')$.

In order to effectively work with pseudodistances, it will be useful
to define the dual set $\sT^*_f$ associated to $f$ as follows:
$\sT^*_f = \curl*{\theta \colon \forall z \in \Rset^{d}, \tri*{\theta,
    z} \leq f(z)}$. This coincides with the traditional notion of
duality in convex analysis; for example, when $f$ is a norm, $\sT^*_f$
coincides with the dual ball of radius one: $\sT^*_f = \curl*{\theta
  \colon \norm*{\theta}_* \leq 1}$ (e.g., when $f$ is the
$d$-dimensional $\ell_{\infty}$ norm, $\sT_{f}^*$ is the
$d$-dimensional $\ell_1$-ball). The following theorem relates the
pseudodistance between $z$ and a convex set $\sS$ to a convex
optimization problem over the dual set.


\begin{restatable}{theorem}{GeneralFunction}
\label{th:general-function}
For any closed convex set $\sS \subset \Rset^d$, the following
equality holds for any $z \in \Rset^d$:
\[
d_f(z, \sS) 
= \inf_{s \in \sS} f(z - s)
= \sup_{\theta \in \sT^*_f} \curl*{\theta \cdot z
  - \sup_{s \in \sS} \theta \cdot s}.
\]
\end{restatable}

\begin{proof}
We adopt the standard definition and notation in optimization for an
\emph{indicator function $I_{\sK}$} of a set $\sK$: for any $x$,
$I_{\sK}(x) = 0$ if $x$ is in $\sK$, $+\infty$ otherwise.
Define $\wt f$ by $\wt f(s) = f(z - s)$ for all $s \in \Rset^d$ and
set $g = I_\sS$.

By definition, the conjugate function of $f$ is defined by: $\forall y
\in \Rset^d, f^*(y) = \sup_{x \in \Rset^{d}} x \cdot y - f(x)$. Now, if
$y$ is in $\sT^*_f$, then we have $f^*(y) \leq f(x) - f(x) = 0$. Thus,
since $f(0) = 0$, the supremum in the definition of $f^*$ is achieved
for $x = 0$ and $f^*(0) = 0$. Otherwise, if $y \not \in \sT^*_f$,
there exists $x \in \sX$ such that $x \cdot y > f(x)$. 
For that $x$, for any
$t > 0$, by the positive
homogeneity of $f$, we have
$(tx) \cdot y - f(tx) 
= t \paren*{x \cdot y - f(x)} > 0$.
Taking the limit $t \to +\infty$, this shows that $f^*(y) =
+\infty$. Thus, we have $f^*(y) = I_{\sT^*_f}$.

By definition, the conjugate function
$\wt f^*$ is defined for all $\theta$
by
$\wt f^*(\theta) 
= \sup_{x \in \sX} x \cdot \theta - f(z - x)
= \sup_{u \in \sX} (z - u) 
\cdot \theta - f(u)
= f^*(-\theta) + z \cdot \theta
= I_{\sT^*_f}(-\theta) + z \cdot \theta$,
which can also be derived from the conjugate function calculus in Table B.1 of
\citep{MohriRostamizadehTalwalkar2018}.

It is also known that the
conjugate function of the indicator function $g$ is defined by $g^*
\colon \theta \mapsto \sup_{s \in \sS} \theta \cdot s$
\citep{BoydVandenberghe2014}.
Since $\dom(\wt f) = \Rset^d$ and $\cont(g) = \sS$, we have
$\dom(\wt f) \cap
\cont(g) = \sS \neq \emptyset$.
Thus, for any convex and bounded set $\sS$ in $\Rset^d$, by Fenchel
duality (Theorem~\ref{th:fenchel-duality},
Appendix~\ref{app:fenchel-duality}), we can write:
\begin{align*}
d_{f}(z, \sS) 
& = \inf_{s \in \sS} f(z - s)\\
& = \inf_{s \in \Rset^d} \curl*{f(z - s) + I_\sS(s)}
\tag{def. of $I_\sS$}\\
& = \sup_{\theta \in \Rset^d} \curl*{- \paren*{I_{\sT^*_f}(-\theta)
    + \theta \cdot z} - \sup_{s \in \sS} \curl*{-\theta \cdot s}}
\tag{Fenchel duality theorem}\\
& = \sup_{-\theta \in \sT^*_f } \curl*{- \theta \cdot z
  - \sup_{s \in \sS} \curl*{-\theta \cdot s}}
\tag{def. of $I_{\sT^*_f}$}\\
& = \sup_{\theta \in \sT^*_f } \curl*{\theta \cdot z
  - \sup_{s \in \sS} \theta \cdot s}.
\tag{change of variable}
\end{align*}
This completes the proof.
\end{proof}

We will primarily be concerned with the case where $\sS$ is a convex
cone. A \textit{convex cone} is a set $\sC$ such that if $x \in \sC$,
$\alpha x \in \sC$ for all $\alpha \geq 0$. In this case, Theorem
\ref{th:general-function} can be simplified to write $d_f(z, \sC)$ as
a \textit{linear} optimization problem over the intersection of
$\sT^{*}_f$ and the polar cone of $\sC$.

\begin{corollary}
\label{cor:quasidistance_olo}
Let $\sC$ be a convex cone, and let $\sC^{\circ} = \{y\,\mid\,\langle
y, x \rangle \leq 0\, \forall x \in \sC\}$ be the polar cone of
$\sC$. Fix a $d$-dimensional pseudonorm $f$ and let $\sT^{\sC}_{f} =
\sT^{*}_f \cap \sC^{\circ}$. Then for any $z \in \Rset^d$,

$$d_f(z, \sC) = \sup_{\theta \in \sT^{\sC}_{f}} (\theta \cdot z).$$
\end{corollary}

\subsection{From high-dimensional $\ell_{\infty}$
  approachability to low-dimensional pseudonorm approachability}
\label{sec:dimension_reduction}

In Section \ref{sec:approachability}, we expressed the regret for a
general regret minimization problem as the $\ell_{\infty}$ distance

\begin{equation}\label{eq:reg_inf}
\Reg = T \cdot d_{\infty}\left(\frac{1}{T}\sum_{t=1}^{T} u(p_t,
\ell_t), (-\infty, 0]^{d}\right).
\end{equation}

\noindent
In this section, we will demonstrate how to rewrite this in terms of a
lower-dimensional pseudodistance.

Consider the bilinear ``basis map'' $\wtu: \sP \times \sL \to
\mathbb{R}^{nm}$ given by $\wtu(p, \ell)_{i, j} = p_{i}\ell_{j}$. Note
that every bilinear function $b\colon \sP \times \sL \to
\mathbb{R}$ can be written in the form $b(p, \ell) = \langle \wtu(p,
\ell), v\rangle$ for some vector $v \in \mathbb{R}^{nm}$ (i.e., the
monomials in $\wtu$ form a basis for the set of biaffine functions on
$\sP \times \sL$)\footnote{If it is helpful, one can think of $\wtu$
as the natural map from $\sP \times \sL$ to the tensor product
$\sP\,\otimes\,\sL$. The vectors $v_i$ can then be thought of as
elements of the dual space, each representing a linear functional on
$\sP\,\otimes\,\sL$.}.

Let $v_i \in \mathbb{R}^{nm}$ be the vector of coefficients
corresponding to the component $u_i$ of $u$ and consider the function
$f(x) = \max(\max_{i \in [d]} \tri{v_i, x}, 0)$. Note that $f$ is a
pseudonorm on $\mathbb{R}^{nm}$; indeed, it's straightforward to see
that $f$ is non-negative and $f(\alpha x) = \alpha f(x)$ for any
$\alpha \geq 0$. In addition, let $\sC$ be the convex cone defined by
$\sC = \{x\,\mid\, \langle x, v_i \rangle \leq 0\, \forall i \in
[d]\}$. We claim that that we can rewrite the $\ell_{\infty}$ distance
in \eqref{eq:reg_inf} in the following way.

\begin{theorem}
  \label{thm:dimension_reduction}
We have that
\begin{equation}
  d_{\infty}\left(\frac{1}{T}
  \sum_{t=1}^{T} u(p_t, \ell_t), (-\infty, 0]^{d}\right)
    = d_{f}\left(\frac{1}{T}\sum_{t=1}^{T} \wtu(p_t, \ell_t), \sC\right).
\end{equation}
\end{theorem}
The dual set $\sT^*_{f}$ associated to this pseudonorm is given by:
\[
\sT^*_{f} = \curl*{\wttheta \in \Rset^{nm}\colon \forall x \in
  \Rset^{nm}, \tri{\wttheta, x} \leq \max_{i \in [d]} \tri{v_i, x}}.
\]

The following two properties of this dual set will be useful in the
sections that follow. First, we show that we can alternately think of
$\sT^*_{f}$ as the convex hull of the $v_i$.

\begin{lemma}
\label{lemma:DualSet}
The dual set $\sT^*_{f}$ coincides with the convex hull of $v_i$s:
$\sT^*_{f}
= \conv \curl*{v_1, \ldots, v_d}$.
\end{lemma}

Secondly, we show that the dual set $\sT_{f}^*$ is contained within
the polar cone $\sC^{\circ}$ of $\sC$.

\begin{lemma}
\label{lemma:containment}
We have that $\sT_{f}^* \subseteq \sC^{\circ}$, where $\sC^{\circ} =
\{y\,\mid\, \langle y, x\rangle \leq 0 \, \forall x \in \sC\}$.
\end{lemma}
\begin{proof}
By Lemma \ref{lemma:DualSet}, it suffices to show that each $v_i \in
\sC^{\circ}$, i.e., for each $v_i$, that $\langle v_i, x \rangle \leq
0$ for all $x \in \sC$. However, this immediately follows by the
definition of $\sC$, since each $x \in \sC$ satisfies $\langle x, v_i
\rangle \leq 0$ for all $i \in [d]$. (An equivalent way of thinking
about this is that $\sC$ is already the polar cone of the cone
generated by the $v_i$, so the $v_i$ must lie in the polar to $\sC$).
\end{proof}

This allows us to simplify Corollary \ref{cor:quasidistance_olo} even further.

\begin{corollary}
\label{cor:regret_olo}
For this convex cone $\sC$ and pseudonorm $f$, we have that for any $z
\in \Rset^{mn}$,

$$d_{f}(z, \sC) = \sup_{\wttheta \in \sT_{f}^*}
\left\langle \wttheta, z \right\rangle.$$
\end{corollary}

Finally, we prove the following ``separability'' conditions for this
approachability problem.

\begin{lemma}\label{lemma:low-dim-separability}
The following two statements are true.

\begin{enumerate}
    \item For any $\ell \in \sL$, there exists a $p \in \sP$ such that
      $\wtu(p, \ell) \in \sC$.
      
    \item For any $\wttheta \in \sT_{f}^*$, there exists a $p \in \sP$
      such that $\tri{\wttheta, \wtu(p, \ell)} \leq 0$ for all $\ell
      \in \sL$.
\end{enumerate}
\end{lemma}

\subsection{Algorithms for pseudonorm approachability}
\label{sec:pseudonorm_alg}

We now present an algorithm for pseudonorm approachability in the
setting of Section \ref{sec:dimension_reduction} (i.e., for the
bilinear function $\wtu$ and the convex cone $\sC$). Just as the
algorithm in Section \ref{sec:approachability} for
$\ell_{\infty}$-approachability required an OLO algorithm for the
simplex, this algorithm will assume we have access to an OLO algorithm
$\wtcF$ for the sets $\sX = \sT_{f}^*$, and $\sY = -\conv(\wtu(\sP,
\sL))$ (recall that $\conv(\wtu(\sP, \sL))$ denotes the convex hull of
the points $\wtu(p, \ell) \in \Rset^{nm}$ for $p \in \sP$ and $\ell
\in \sL$).


\begin{algorithm2e}
\SetAlgoLined
Initialize $\wttheta_1$ to an arbitrary point in $\sT_{f}^*$\;
\BlankLine
\For{$t\leftarrow 1$  \KwTo $T$}{
Choose $p_t \in \sP$ so that for all $\ell \in \sL$, $\langle \wttheta_t, \wtu(p_t, \ell)\rangle \leq 0$ (such a $p_t$ is guaranteed to exist by the second statement in Lemma \ref{lemma:low-dim-separability})\;
Play action $p_t$ and receive as feedback $\ell_t \in \sL$\;
Set $y_t = -\wtu(p_t, \ell_t)$\;
Set $\wttheta_{t+1} = \wtcF(y_1, y_2, \dots, y_t)$.
}
\caption{Description of Algorithm $\wtcA$ for $f$-approachability}
\label{algo:algorithm-pseudonorm-approach}
\end{algorithm2e}

Our algorithm $\wtcA$ is summarized above in Algorithm
\ref{algo:algorithm-pseudonorm-approach}. We now have the following
analogue of Theorem \ref{thm:linf_approachability}.

\begin{theorem}
\label{thm:pseudonorm_approachability}
The following guarantee holds for the pseudonorm approachability algorithm
$\wtcA$:
$$\Reg(\wtcA) = T \cdot d_f\left(\frac{1}{T}\sum_{t=1}^{T}\wtu(p_{t},
\ell_t), \sC \right) \leq \Reg(\wtcF).$$
\end{theorem}
\begin{proof}
The first equality follows as a direct consequence of Theorem
\ref{thm:dimension_reduction} (which proves that the $d_f$ distance to
$\sC$ is equal to the analogous $\ell_{\infty}$ distance to $(-\infty,
0]^d$) and Theorem \ref{thm:linf_approachability} (which shows that
  this $\ell_{\infty}$ distance is equal to the regret of our regret
  minimization problem). It therefore suffices to prove the second
  equality.

Note that
\begin{eqnarray*}
d_f\left(\frac{1}{T}\sum_{t=1}^{T}\wtu(p_{t}, \ell_t), \sC \right) &=& \sup_{\wttheta \in \sT_{f}^*} \left\langle \wttheta, \frac{1}{T}\sum_{t=1}^{T}\wtu(p_{t}, \ell_t)\right\rangle\\
&=& -\inf_{\wttheta \in \sX} \left\langle \wttheta, \frac{1}{T}\sum_{t=1}^{T}\left(-\wtu(p_{t}, \ell_t)\right)\right\rangle\\
&=& \frac{1}{T}\Reg(\wtcF) - \frac{1}{T}\sum_{t=1}^{T}\left\langle \wttheta_t, -\wtu(p_{t}, \ell_t)\right\rangle \\
&\leq& \frac{1}{T}\Reg(\wtcF).
\end{eqnarray*}

Here, the first equality holds as a consequence of Corollary
\ref{cor:regret_olo}, and the last inequality holds since (by choice
of $p_t$ in step 2a) $\langle \wttheta_t, \wtu(p_t, \ell)\rangle \leq
0$ for all $t \in [T]$.
\end{proof}

If we choose $\wtcF$ to be FTRL with a quadratic regularizer, Lemma
\ref{lem:quadratic_regularizer} implies the following result.

\begin{theorem}
\label{thm:poly_regret}
Let $u\colon \sP \to \sL$ be a bilinear regret function. Then
there exists a regret minimization algorithm $\wtcA$ for $u$ with
regret

$$\Reg(\wtcA) = O(D_{z}(\diam \sP)(\diam \sL)\sqrt{T}),$$

\noindent
where $D_{z} = \max_{i \in [d]} ||v_i||$. If we let 

$$\lambda = \max_{i \in [d]} ||v_{i}||_{\infty} = \max_{i \in [d], j \in [m], k \in [n]} \left| \frac{\partial^2 u_i}{\partial p_j \partial \ell_k} \right|,$$

\noindent
then this regret bound further satisfies

$$\Reg(\wtcA) = O(\lambda nm \sqrt{T}).$$
\end{theorem}
\begin{proof}
To see the first result, note that
Lemma~\ref{lem:quadratic_regularizer} directly implies a bound of
$O(\diam(\sX)\diam(\sY)\sqrt{T})$, where $\sX = \sT_{f}^{*}$ and $\sY
= -\conv(\wtu(\sP, \sL))$. We'll now proceed by simplifying
$\diam(\sX)$ and $\diam(\sY)$. First, for $\sX$, recall from Lemma
\ref{lemma:DualSet} that $\sX = \sT_{f}^*$ is the convex hull of the
$d$ vectors $v_i \in \Rset^{nm}$, so $\diam(\sX) = O(\max_{i \in [d]}
||v_i||) = O(D_z)$. Second, for $\sY$, note that $\diam(\sY) =
O(\max_{y \in \sY} ||y||) = O(\max_{p \in \sP, \ell \in \sL}||\wtu(p,
\ell)||)$. But $||\wtu(p, \ell)|| = ||p|| \cdot ||\ell||$, so
$\diam(\sY) = O\left(\max_{p \in \sP}||p|| \cdot \max_{\ell \in
  \sL}||\ell||\right) = O((\diam\sP)(\diam \sL))$. Combining these, we
obtain the first inequality.

The second result directly follows from the following three facts:
i. $\diam \sP = O(\sqrt{n})$ (since $\sP \subseteq [-1, 1]^n$,
ii. $\diam \sL = O(\sqrt{m})$ (since $\sL \subseteq [-1, 1]^m$), and
iii. $D_z \leq \lambda \sqrt{mn}$.
\end{proof}

Theorem \ref{thm:poly_regret} shows that in settings where $\lambda$
is constant (which is true in all the settings we consider), there
exists an algorithm where $\Reg(\wtcA) = \poly(n, m) \sqrt{T}$;
notably, $\Reg(\wtcA)$ does not depend on the dimension $d$, which in
many cases can be thought of as the number of benchmarks of comparison
for our learning algorithm.

In the following two sections, we will show how to strengthen this
result in two different ways. First, we will show that (modulo some
fairly mild computational assumptions) it is possible to efficiently
implement the algorithm of Theorem \ref{thm:poly_regret}. Second, we
will show that by using a different choice of regularizer, we can
recover exactly the regret obtained in Corollary
\ref{cor:linf_regret_minimization}.

\subsection{Efficient algorithms for pseudonorm approachability}

\subsubsection{Computational assumptions}

In this section we will discuss how to transform the algorithm in
Theorem \ref{thm:poly_regret} into a computationally efficient
algorithm. Note that without any constraints on $\sL$, $\sP$, and $u$
or how they are specified, performing any sort of regret minimization
efficiently is a hopeless task (e.g., consider the case where it is
computationally hard to even determine membership in $\sP$). We'll
therefore make the following three structural / computational
assumptions on $\sL$, $\sP$, and $u$.

First, we will restrict our attention to cases where the loss set
$\sL$ is ``orthant-generating''. A convex subset $S$ of $\Rset^d$ is
\textit{orthant-generating} if $S \subseteq [0, \infty)^d$ and for
  each $i \in [d]$, there exists an $\lambda_i > 0$ such that
  $\lambda_ie_i \in S$ (as part of this assumption, we will also
  assume we have access to the values $\lambda_i$). Note that many
  common choices for $\sL$ (e.g., $[0, 1]^m$, $\Delta_m$,
  intersections of other $\ell_p$ balls with the positive orthant) are
  all orthant-generating.

Second, we will assume we have an \textit{efficient separation oracle}
for the action set $\sP$; that is, an oracle which takes a $p \in
\Rset^n$ and outputs (in time $\poly(n)$) either that $p \in \sP$ or a
separating hyperplane $h \in \Rset^n$ such that $\langle p, h \rangle
> \max_{p' \in \sP} \langle p', h \rangle$.

Finally, we will assume we have access to what we call an
\textit{efficient regret oracle} for $u$. Given a collection of $R$
action/loss pairs $(p_r, \ell_{r}) \in \sP \times \sL$ and $R$
positive constants $\alpha_r \geq 0$, an efficient regret oracle can
compute (in time $\poly(n, m, R)$) the value of $\max_{i \in [d]}
\sum_{r=1}^{R} \alpha_r u_i(p_r, \ell_r)$. This can be thought of as
evaluating the function $\Reg(\bm{p}, \bm{\ell})$ for a pair of action
and loss sequences that take on the action/loss pair $(p_r, \ell_r)$
an $\alpha_r/\sum_{r'}\alpha_{r'}$ fraction of the time. At a higher
level, having access to an efficient regret oracle means that a
learner can efficiently compute their overall regret at the end of $T$
rounds (it is hard to imagine how one could efficiently minimize
regret without being able to compute it).

\subsubsection{Extending the dual set}

One of the ingredients we will need to implement algorithm $\wtcA$ is
a membership oracle for the dual set $\sT_{f}^*$ (e.g.\ in order to
perform OLO over this dual set). To check whether $\wttheta \in
\sT_{f}^*$, it suffices to check whether $\langle \wttheta, x\rangle
\leq \max_{i \in [d]} \langle v_{i}, x \rangle$ for all $x \in
\Rset^{nm}$, so in turn, it will be useful to be able to compute the
function $\max_{i \in [d]} \langle v_i, x \rangle$ for any $x \in
\Rset^{nm}$.

Computing this maximum is \textit{very similar}\footnote{In fact, in
almost all practical cases where we have an efficient regret oracle,
it is possible by a similar computation to directly compute this
maximum. Here we describe an approach that works in a blackbox way
given only a strict regret oracle -- if you are accept the existence
of an oracle that can optimize linear functions over the $v_i$, you
can skip this subsection.} to what is provided by our regret oracle:
note that by writing $u_i(p_r, \ell_r)$ as $\langle v_i, \wtu(p_r,
\ell_r)\rangle$, we can think of the regret oracle providing the value
of:

$$\max_{i \in [d]} \left\langle v_i, \sum_{r=1}^{R}\alpha_r \wtu(p_r,
\ell_r) \right\rangle.$$

If it were possible to write any $x \in \Rset^{d}$ in the form
$\sum_{r=1}^{R}\alpha_r \wtu(p_r, \ell_r)$, we would be done. However,
this may not be possible: in particular, $\sum_{r=1}^{R}\alpha_r
\wtu(p_r, \ell_r)$ must lie in the convex cone $\cone(\wtu(\sP, \sL))$
generated by all points of the form $\wtu(p, \ell)$.

We will therefore briefly generalize the theory of Section
\ref{sec:pseudonorm_approachability} to cases where we are only
able to optimize over an extension of the dual set. Given a convex
cone $\sZ$, let $\sT_{f}^{*}(\sZ) = \{\theta\,\mid\, \forall z \in
\sZ, \tri{\theta, z} \leq f(z)\}$. Note that $\sT_{f}^{*} \subseteq
\sT_{f}^{*}(\sZ)$ (since in $\sT_{f}^*$, $\tri{\theta, z} \leq f(z)$
must hold for all $z$ in the ambient space). The following lemma shows
that if $z$ is in $\sZ$, to maximize a linear function over
$\sT_{f}^*$ it suffices to maximize it over $\sT_{f}^*(\sZ)$.

\begin{lemma}
\label{lem:dual_extension}
Let $\sZ$ be an arbitrary convex cone. Then if $z \in \sZ$, the
following equalities hold:
\[
\sup_{\theta \in \sT^{*}_f} \tri{\theta, z}
= \sup_{\theta \in \sT^{*}_f(\sZ)} \tri{\theta, z} = f(z).
\]
\end{lemma}

Now, consider the specific convex cone $\sZ = \cone(\wtu(\sP,
\sL))$. Given Lemma \ref{lem:dual_extension}, it is straightforward to
check that Theorem \ref{thm:pseudonorm_approachability} continues to
hold even if the domain $\sX$ of the OLO algorithm $\wtcF$ is set to
$\sT_{f}^*(\sZ)$ instead of $\sT_{f}^*$. In particular, the first
equality in the proof of Theorem \ref{thm:pseudonorm_approachability}
is still true when $\sT_{f}^*$ is replaced by $\sT_{f}^*(\sZ)$, since
$\frac{1}{T}\sum_{t}\wtu(p_t, \ell_t) \in \sZ$.

There is one other issue we must deal with: it is possible that
$\sT_{f}^*(\sZ)$ is significantly larger than $\sT_{f}^*$, and
therefore an OLO algorithm with domain $\sT_{f}^*(\sZ)$ might incur
more regret than that of $\sT_{f}^{*}$. In fact, there are cases where
the set $\sT_{f}^*(\sZ)$ is unbounded. Nonetheless, the following
lemma shows that for OLO with a quadratic regularizer, we will never
encounter very large values of $\wttheta$.

\begin{lemma}
\label{lemma:bound_extension}
Let $R(\wttheta) = ||\wttheta||^2$, and fix a $z \in \sZ$ and $\eta >
0$. Let

$$\wttheta_{opt} = \argmin_{\wttheta \in \sT_{f}^*(\sZ)}\left(\eta
R(\wttheta) - \langle \wttheta, z \rangle\right).$$

\noindent
Then $||\wttheta_{opt}|| \leq \diam(\sT_{f}^*)$.
\end{lemma}

Lemma \ref{lemma:bound_extension} therefore implies that a
quadratic-regularized OLO algorithm $\wtcF$ will only output values of
$\wttheta$ with $||\wttheta|| \leq \diam(\sT_{f}^*)$, and therefore we
can safely set $\sX = \sT_{f}^* \cap \{\theta \mid ||\theta|| \leq
\rho\}$, where $\rho = \Theta(\diam(\sT_{f}^*))$ is an efficiently
computable upper bound on $\diam(\sT_{f}^*)$. This guarantees that
$\diam(\sX) = O(\diam(\sT_{f}^*))$, and thus the regret guarantees of
Theorem \ref{thm:poly_regret} remain unchanged.

\subsubsection{Constructing a membership oracle}

We will now demonstrate how to use our regret oracle to construct a
membership oracle for the expanded set $\sX$ defined in the previous
section. We first show that it is possible to check for membership in
$\sZ$ (and further, when $z \in \sZ$, write $z$ as a convex
combination of the generators of $\sZ$).

\begin{lemma}
\label{lem:z_membership}
Given a point $z \in \Rset^{nm}$, we can efficiently check whether $z
\in \sZ$. If $z \in \sZ$, we can also efficiently write $z$ in the
form $\sum_{k=1}^{m}\alpha_k \wtu(p_k, \ell_k)$ (for an explicit
choice of $p_k$, $\ell_k$, and $\alpha_k$).
\end{lemma}

Note that expressing $z$ in the form $\sum_{k=1}^{m}\alpha_k \wtu(p_k,
\ell_k)$ allows us to directly apply our efficient regret oracle (with
$R = m$). We therefore gain the following optimization oracle as a
corollary.

\begin{corollary}
\label{cor:eval_max}
Given an efficient regret oracle, for any $z \in \sZ$ we can efficiently (in time $\poly(n, m)$) compute $\max_{i \in [d]} \langle v_i, z \rangle$.
\end{corollary}

Finally, we will show that given these two results, we can efficiently
construct a membership oracle for our set $\sX$. To do this, we will
need the following fact (loosely stated; see Lemma
\ref{lemma:convex-opt} in Appendix \ref{app:fenchel-duality} for an
accurate statement): it is possible to minimize a convex function over
a convex set as long as one has an evaluation oracle for the function
and a membership oracle for the convex set.

\begin{lemma}
\label{lem:x_membership}
Given an efficient regret oracle for $u$, we can construct an efficient membership oracle for the set $\sX$.
\end{lemma}

\subsubsection{Implementing regret minimization}

Equipped with this membership oracle, we can now state and prove our main theorem.

\begin{theorem}
\label{thm:efficient_alg}
Assume that:

\begin{enumerate}
    \item The convex set $\sL$ is orthant-generating.
    \item We have an efficient separation oracle for the convex set $\sP$.
    \item We have an efficient regret oracle for the regret function $u$.
\end{enumerate}

\noindent
Then it is possible to implement the algorithm of Theorem \ref{thm:poly_regret} in $\poly(n, m)$ time per round.
\end{theorem}
\begin{proof}
There are two steps of unclear computational complexity in the
description of the algorithm in Section \ref{sec:pseudonorm_alg}: step
2a, where we find a $p_t$ such that $\langle \wttheta_t, \wtu(p_t,
\ell)\rangle \leq 0$ for all $\ell \in \sL$, and step 2d, where we
have to run our OLO subalgorithm over $\sT_{f}^*$ (which we
specifically instantiate as FTRL with a quadratic regularizer).

We begin by describing how to perform step 2a efficiently. Fix a
$\wttheta_t$. Note that since $\sL$ is orthant-generating, to check
whether $\langle \wttheta_t, \wtu(p_t, \ell)\rangle \leq 0$ for all
$\ell \in \sL$, it suffices to check whether $\wtu(p_t, e_k) \leq 0$
for each unit vector $e_k \in \Rset^{m}$. Therefore, to find a valid
$p_t$ we must find a point $p_t \in \sP$ that satisfies an additional
$m$ explicit linear constraints. Since we have an efficient separation
oracle for $\sP$, this is possible in $\poly(n, m)$ time.

To implement the FTRL algorithm $\cF$ over the set $\sX = \sT_{f}^*$
with a quadratic regularizer, each round we must find the minimizer of
the convex function $g(\wttheta) = R(\wttheta) + \eta\sum_{s=1}^{t-1}
\langle \wttheta, \wtu(p_s, \ell_s)\rangle$ over the convex set
$\sX$. To do this, it suffices to exhibit an evaluation oracle for $g$
and a membership oracle for $\sX$. To evaluate $g$, note that we
simply need to be able to compute $R(\wttheta)$ (which is a Euclidean
distance in $\Rset^{mn}$) and $\sum_{s=1}^{t-1} \langle \wttheta,
\wtu(p_s, \ell_s)\rangle$ which we can do in $O(mn)$ time by keeping
track of the cumulative sum $\sum_{s=1}^{t-1} \wtu(p_s, \ell_s)$ and
computing a single inner product in $\Rset^{mn}$. A membership oracle
for $\sX$ is provided by Lemma \ref{lem:x_membership}.
\end{proof}

\subsection{Recovering $\ell_{\infty}$ regret via maxent regularizers}

One interesting aspect of this reduction to pseudonorm approachability
is that, in some cases, the regret bound of $O(\sqrt{T\log d})$
achievable via $\ell_{\infty}$-approachability (Corollary
\ref{cor:linf_regret_minimization}) sometimes outperforms the regret
bound of $O(\lambda mn\sqrt{T})$ achieved by pseudonorm
approachability (Theorem \ref{thm:poly_regret}), for example when $d =
\poly(m, n)$. Of course, this comparison is not completely fair: in
both cases there is flexibility to specify the underlying OLO
algorithm, and the $O(\sqrt{T\log d})$ bound uses an FTRL algorithm
with a negative entropy regularizer, whereas our pseudonorm
approachability bound uses an FTRL algorithm with a quadratic
regularizer. After all, there are well-known cases (e.g.\ for OLO over
a simplex, as in $\ell_{\infty}$-approachability) where the negative
entropy regularizer leads to exponentially better (in $d$) regret
bounds than the quadratic regularizer.

In this subsection, we will show that there exists a different
regularizer for pseudonorm approachability -- one we call a
\textit{maxent regularizer} -- which recovers the $O(\sqrt{T\log d})$
regret bound of Corollary \ref{cor:linf_regret_minimization}.  In
doing so, we will also better understand the parallels between the
regret minimization algorithm $\cA$ (which works via reduction to
$\ell_{\infty}$-approachability in a $d$-dimensional space) and the
regret minimization algorithm $\wtcA$ (which works via reduction to
pseudonorm approachability in an $nm$-dimensional space).

Let $V$ be the $d$-by-$nm$ matrix whose $i$th row equals $v_i$. Note
that $V$ allows us to translate between analogous concepts/quantities
for $\cA$ and $\wtcA$ in the following way.

\begin{lemma}
  \label{lemma:translation}
The following statements are true:

\begin{itemize}

\item For any $p \in \sP$ and $\ell \in \sL$, $u(p, \ell) = V\wtu(p, \ell)$.

\item The dual set $\sT_{f}^* = V^{T}\Delta_d$ (i.e., $\wttheta \in
  \sT_{f}^*$ iff there exists a $\theta \in \Delta_d$ such that
  $\wttheta = V^{T}\theta$).

\item If $\theta \in \Delta_d$ and $\wttheta = V^{T}\theta \in
  \sT_{f}^*$, then for any $p \in \sP$ and $\ell \in \sL$,
  $\tri{\theta, u(p, \ell)} = \tri{\wttheta, \wtu(p, \ell)}$.

\item Fix $\theta_t \in \Delta_d$ and let $\wttheta = V^{T}\theta$. If
  $p_t$ satisfies $\tri{\theta_t, u(p_t, \ell)} \leq 0$ for all $\ell
  \in \sL$, then $p_t$ also satisfies $\tri{\wttheta_t, \wtu(p_t,
    \ell)} \leq 0$ for all $\ell \in \sL$.
\end{itemize}
\end{lemma}
\begin{proof}
The first statement follows from the fact that $V\wtu(p, \ell)_i =
\langle v_i, \wtu(p, \ell)\rangle = u(p, \ell)_i$. The second claim
follows as a consequence of Lemma \ref{lemma:DualSet} (that
$\sT_{f}^*$ is the convex hull of the vectors $v_i$). The third claim
follows from the first claim: $\langle \theta, u(p, \ell)\rangle =
\langle \theta, V\wtu(p, \ell) \rangle = \langle V^{T}\theta, \wtu(p,
\ell) \rangle = \langle \wttheta, \wtu(p, \ell) \rangle$. Finally, the
fourth claim follows directly from the third claim.
\end{proof}

Now, fix a regret minimization problem (specified by $\sP$, $\sL$, and
$u$) and consider the execution of algorithm $\cA$ on some specific
loss sequence $\bm{\ell} = (\ell_1, \ell_2, \dots, \ell_{T})$. Each
round $t$, $\cA$ runs the entropy-regularized FTRL algorithm $\cF$ to
generate a $\theta_t \in \Delta_d$ (as a function of the actions and
losses up until round $t$) and then uses $\theta_t$ to select a $p_t$
that satisfies $\tri{\theta_t, u(p_t, \ell)} \leq 0$ for all $\ell \in
\sL$. If we execute algorithm $\wtcA$ for the same regret minimization
problem on the same loss sequence, each round $t$, $\wtcA$ runs some
(to be determined) FTRL algorithm $\wtcF$ to generate a $\wttheta_t
\in \sT_{f}^*$, and then uses $\wttheta_t$ to select a $p_t$ that
satisfies $\tri{\wttheta_t, \wtu(p_t, \ell)}$ for all $\ell \in
\sL$. Lemma \ref{lemma:translation} shows that if $\wttheta_t =
V^{T}\theta_t$ for each $t$, then both algorithms will generate
\textit{exactly the same} sequence of actions in response to this loss
sequence\footnote{Technically, there may be some leeway in terms of
which $p_t$ to choose that e.g. satisfies $\tri{\theta_t, u(p_t,
  \ell)} \leq 0$ for all $\ell \in \sL$, and the two different
procedures could result in different choices of $p_t$. But if we break
ties consistently (e.g. add the additional constraint of choosing the
$p_t$ that maximizes the inner product of $p_t$ with some generic
vector), then both procedures will produce the same value of $p_t$.},
and hence the same regret.

The question then becomes: how do we design an OLO algorithm $\wtcF$
that outputs $\wttheta_t = V^{T}\theta_t$ each round that $\cF$ would
output $\theta_t$? Recall that if $\cF$ is an FTRL algorithm with
regularizer $R:\Delta_d \to \Rset$, then

\[
\theta_t = \argmin_{\theta \in \Delta_d}\left(R(\theta) +
\sum_{s=1}^{t-1}\langle \theta, u(p_t, \ell_t)\rangle\right).
\]
Define $\wtR:\sT_{f}^* \to \Rset$ via

\begin{equation}
    \wtR(\wttheta) = \min_{\substack{\theta \in \Delta_d\\V^{T}\theta = \wttheta}}R(\theta).
\end{equation}

We claim that if we let $\wtcF$ be the FTRL algorithm that uses
regularizer $\wtR$, then $\wtcF$ will output our desired sequence of
$\wttheta_t$.

\begin{lemma}
The following equality holds for the output $\wttheta_t$:
$$\wttheta_t = \argmin_{\wttheta \in \sT_{f}^*}\left(\wtR(\wttheta) +
\sum_{s=1}^{t-1}\langle \wttheta, \wtu(p_t, \ell_t)\rangle\right) =
V^{T}\theta_t.$$
\end{lemma}
\begin{proof}
Note that we can write

\begin{eqnarray*}
\wttheta_t &=& \argmin_{\wttheta \in \sT_{f}^*}\left(\wtR(\wttheta) + \sum_{s=1}^{t-1}\langle \wttheta, \wtu(p_t, \ell_t)\rangle\right)\\
&=& \argmin_{\wttheta \in \sT_{f}^*}\left(\left(\min_{\substack{\theta \in \Delta_d\\V^{T}\theta = \wttheta}}R(\theta)\right) + \sum_{s=1}^{t-1}\langle \wttheta, \wtu(p_t, \ell_t)\rangle\right)\\
&=& \argmin_{\wttheta \in \sT_{f}^*}\left(\min_{\substack{\theta \in \Delta_d\\V^{T}\theta = \wttheta}}\left(R(\theta) + \sum_{s=1}^{t-1}\tri{\theta, u(p_t, \ell_t)}\right)\right) \\
&=& V^{T}\argmin_{\theta \in \Delta_d}\left(R(\theta) + \sum_{s=1}^{t-1}\tri{\theta, u(p_t, \ell_t)}\right) \\
&=& V^{T}\theta_t.
\end{eqnarray*}

\noindent
The third equality follows from the fact that if $\wttheta =
V^{T}\theta$, then $\tri{\theta, u(p, \ell)} = \tri{\wttheta, \wtu(p,
  \ell)}$ (by Lemma~\ref{lemma:translation}).
\end{proof}

\begin{corollary}
\label{cor:maxent_regularizer}
Let $\cF$ be an FTRL algorithm over $\Delta_{d}$ with regularizer $R$,
and let $\wtcF$ be an FTRL algorithm over $\sT_{f}^*$ with regularizer
$\wtR$. Then $\Reg(\cA) = \Reg(\wtcA)$.
\end{corollary}

We now consider the specific case where $R$ is the negentropy
function; i.e., for $\theta \in \Delta_{d}$, $R(\theta) = -H(\theta) =
\sum_{i=1}^{d}\theta_i \log \theta_i$. For this choice of $R$, $\wtR$
becomes

\begin{equation}\label{eq:maxent_reg}
    \wtR(\wttheta) = -\max_{\substack{\theta \in \Delta_d\\V^{T}\theta = \wttheta}}H(\theta).
\end{equation}

In other words, $-\wtR(\wttheta)$ is the maximum entropy of any distribution $\theta$ in $\Delta_d$ that satisfies the $nm$ linear constraints imposed by $V^{T}\theta = \wttheta$. This is exactly an instance of the Maxent problem studied in
\citep{BergerDellaPietraDellaPietra1996, 
Rosenfeld1996, 
DellaPietraDellaPietraLafferty1997, 
DudikPhillipsSchapire2007} and
\citep{MohriRostamizadehTalwalkar2018}[chapter~12]. 
\ignore{MM: I think it is fine to comment this out.
Unfortunately, it is not clear whether it is possible to efficiently
solve this problem (i.e., compute $\wtR$); however, there are multiple
things we can say the general solution that may be useful in specific
cases.  }

It is known \citep{DellaPietraDellaPietraLafferty1997,
  DudikPhillipsSchapire2007,MohriRostamizadehTalwalkar2018} that the
entropy maximizing distribution is a \textit{Gibbs distribution}. In
particular, the $\theta$ maximizing the expression in
\eqref{eq:maxent_reg} satisfies (for some real constants
$\lambda_{jk}$ for $j \in [n]$, $k \in [m]$)

\begin{equation}
\label{eq:gibbs}
\theta_i = \frac{\exp\left(\sum_{j=1}^{n}\sum_{k=1}^{m}
  \lambda_{jk}v_{ijk}\right)}{Z(\lambda)},
\end{equation}

\noindent
where $Z(\lambda)$ (the ``partition function'') is defined via 

\begin{equation}
  Z(\lambda) = \sum_{i=1}^{d}\exp\left(\sum_{j=1}^{n}\sum_{k=1}^{m}
  \lambda_{jk}v_{ijk}\right).
\end{equation}
Generally, there is exactly one choice of $\lambda_{jk}$ which results
in $V^{T}\theta = \wttheta$ (since there are $mn$ free variables and
$mn$ linear constraints). For this optimal $\lambda$, it is known that
the maximum entropy is given by $\langle \lambda, \wttheta \rangle -
\log Z(\lambda)$. If it is possible to solve this system for
$\lambda_{jk}$ and evaluate $Z(\lambda)$ efficiently, we can then
evaluate $\wtR$ efficiently. In Section \ref{sec:swap} we will see how
to do this for the specific case of swap regret (where we are helped
by the fact that \eqref{eq:gibbs} guarantees that $\theta$ is a
product distribution).

Regardless of how we compute the maxent regularizer $\wtR$, efficient
computation leads to an efficient $O(\sqrt{T\log d})$ regret
minimization algorithm.

\begin{corollary}
  \label{cor:efficient_entropy}
If there exists an efficient ($\poly(n, m)$ time) algorithm for
computing the maxent regularizer $\wtR(\wttheta)$, then there exists a
regret minimization algorithm $\wtcA$ for $u$ with regret $O(\sqrt{T
  \log d})$ and that can be implemented in $\poly(n, m)$ time per
round.
\end{corollary}

\section{Applications}

\subsection{Swap regret}
\label{sec:swap}

Recall that in the setting of swap regret, the action set is $\sP =
\Delta_K$ (distributions over $K$ actions), the loss set is $\sL = [0,
  1]^K$, and the \textit{swap regret} of playing a sequence of actions
$\bm{p}$ on a sequence of losses $\bm{\ell}$ is given by

$$\SwapReg(\bm{p}, \bm{\ell}) = \max_{\pi: [K] \to [K]}
\sum_{t=1}^{T} \sum_{i=1}^{K}(p_{ti}\ell_{ti} -
p_{ti}\ell_{t\pi(i)}).$$

In words, swap regret compares the total loss achieved by this
sequence of actions with the loss achieved by any transformed action
sequence formed by applying an arbitrary ``swap function'' to the
actions (i.e., always playing action $\pi(i)$ instead of action
$i$). Swap regret minimization can be directly written as an
$\ell_{\infty}$-approachability problem for the bilinear function
$u\colon \sP \times \sL \to \Rset^{K^K}$ where
\[
u_{\pi}(p, \ell) = \sum_{i=1}^{K}(p_i\ell_i - p_i\ell_{\pi(i)}).
\]

\noindent
(here we index $u$ by the $K^K$ functions $\pi: [K] \to
    [K]$). Note that the negative orthant is indeed separable with
    respect to $u$ since if we let $p^*(\ell) = \argmin_{p \in \sP}
    \langle p, \ell\rangle$, $u_{\pi}(p, \ell) \leq 0$ for any swap
    function $\pi$.

We can now apply the theory developed in Section
\ref{sec:main}. First, since $m = n = K$ and the maximum absolute
value of any coefficient in $u$ is $1$, Theorem \ref{thm:poly_regret}
immediately results in a swap regret algorithm with regret
$O(K^2\sqrt{T})$. Moreover, since we can write

$$\max_{\pi: [K] \to [K]} \sum_{t=1}^{T} \sum_{i=1}^{K}
(p_{ti}\ell_{ti} - p_{ti}\ell_{t\pi(i)}) = \sum_{i=1}^{K}
\max_{j \in [K]}\sum_{t=1}^{T} (p_{ti}\ell_{ti} - p_{ti}\ell_{tj}),$$

\noindent
we can compute $\SwapReg$ efficiently. By Theorem
\ref{thm:efficient_alg}, we can therefore implement this
$O(K^2\sqrt{T})$-regret algorithm efficiently (in $\poly(K)$ time per
round). We can improve upon this regret bound by noting that
$\diam(\sP) = O(1)$, $\diam(\sL) = O(\sqrt{K})$, and $D_z =
O(\sqrt{K})$ (the coefficient vector $v_{\pi}$ corresponding to
$u_{\pi}(p, \ell)$ contains at most $2K$ coefficients that are $\pm
1$).  It follows that $D_z(\diam \sP)(\diam \sL) = O(K)$, and it
follows from the first part of Theorem \ref{thm:poly_regret} that the
regret of the aforementioned algorithm is actually only
$O(K\sqrt{T})$.

This is still a factor of approximately $O(\sqrt{K})$ larger than the
optimal $O(\sqrt{TK\log K})$ bound. To achieve the optimal regret
bound, we will show how to compute the maxent regularizer for swap
regret (and hence can efficiently implement an $O(\sqrt{T\log d}) =
O(\sqrt{TK\log K})$ algorithm via Corollary
\ref{cor:efficient_entropy}). Recall that the maxent regularizer
$\wtR(\wttheta)$ is defined for $\wttheta \in V^{T}\Delta_d$ is the
negative of the maximum entropy of a distribution $\theta \in
\Delta_{d}$ that satisfies $\wttheta = V^{T}\theta$. For our problem,
$\wttheta$ is $(K^2)$-dimensional, and this imposes the following
linear constraints on $\theta$ (which we view as a distribution over
swap functions $\pi$):

\begin{eqnarray*}
\wttheta_{ij} &=& - \sum_{\pi \mid \pi(i) = j} \theta_{\pi} \hspace{10mm} \mbox{ for $j \neq i$}\\
\wttheta_{ii} &=& 1 - \sum_{\pi \mid \pi(i) = i} \theta_{\pi}.
\end{eqnarray*}

Now, by the characterization presented in \eqref{eq:gibbs}\footnote{In
particular, in the case of swap regret we have that $v_{\pi, i, j} =
\bm{1}(k = j) - \bm{1}(k = \pi(j))$. Since the first term
($\bm{1}(k=j)$) does not depend on $\pi$, we can ignore its
contribution to $\theta_{\pi}$ (it cancels from the numerator and
denominator).}, we know the entropy maximizing $\theta$ satisfies (for
some $K^2$ constants $\lambda_{ij}$)

\begin{eqnarray*}
\theta_{\pi} = 
\frac{\exp\left(\sum_{i=1}^{K}\lambda_{i \pi(i)}\right)}{\sum_{\pi'}\exp\left(\sum_{i=1}^{K}\lambda_{i \pi'(i)}\right)} 
= \frac{\prod_{i=1}^{K}\exp\left(\lambda_{i \pi(i)}\right)}{\sum_{\pi'}\prod_{i=1}^{K}\exp\left(\lambda_{i \pi'(i)}\right)} 
= \prod_{i=1}^{K} \frac{\exp(\lambda_{i\pi(i)})}{\sum_{j=1}^{K}\exp(\lambda_{ij})}.
\end{eqnarray*}

In particular, this shows that the entropy maximizing distribution
$\theta$ is a product distribution over the set of swap functions
$\pi$ where for each $i \in [K]$ the value of $\pi(i)$ is chosen
independently. Moreover, from $\wttheta$ we can recover the overall
marginal probability $q_{ij} = \Pr_{\pi \sim \theta}[\pi(i) = j]$ (it
is $-\wttheta_{ij}$ if $j \neq i$ and $1 - \wttheta_{ii}$ if $j =
i$). The entropy of this product distribution can therefore be written
as:
\[
H(\theta) = \sum_{i=1}^{K} H(\pi (i)) = -\sum_{i=1}^{K} \sum_{j=1}^{K}
q_{ij}\log q_{ij}.
\]
Our regularizer $\wtR(\wttheta)$ is simply $-H(\theta)$ and can
clearly be efficiently computed as a function of $\wttheta$. It
follows from Corollary \ref{cor:efficient_entropy} that there exists
an efficient ($\poly(K)$ time per round) regret minimization algorithm
for swap regret that incurs $O(\sqrt{T\log d}) = O(\sqrt{TK\log K})$
worst-case regret.

Finally, we briefly remark that the pseudonorm $f$ we construct here
is closely related to the $\ell_{1, \infty}$ \textit{group norm}
defined over $K$-by-$K$ square matrices as the $\ell_{1}$ norm of
vector formed by the $\ell_{\infty}$ norms of the rows (i.e.,
$||M||_{1, \infty} = \sum_{i=1}^{K}\max_{j \in [K]} |M_{ij}|$). This
is not unique to swap regret; in many of our applications, the
relevant pseudonorm can be thought of as a composition of multiple
smaller norms (often $\ell_1$ or $\ell_{\infty}$ norms).

\subsection{Procrustean swap regret}

To illustrate the power of Theorem \ref{thm:efficient_alg}, we present
a toy variant of swap regret where the learner must compete against an
\emph{infinite} set of swap functions (in particular, all orthogonal
linear transformations of their sequence of actions) and yet can do
this efficiently while incurring low (polynomial in the dimension of
their action set) regret.

In this problem, the action set $\sP = \{x \in \Rset^{n} \,;\,
||x||_{2} \leq 1\}$ is the set of unit vectors with norm at most 1 and
the loss set $\sL = [-1, 1]^n$. The learner would like to minimize the
following notion of regret:
\begin{equation}
  \label{eq:procrustes}
  \PReg(\bm{p}, \bm{\ell}) = \max_{Q \in O(n)}
  \sum_{t=1}^{T} (\langle p_t, \ell_t \rangle - \langle Qp_t, \ell_t \rangle).
\end{equation}
Here, $O(n)$ is the set of all orthogonal $n$-by-$n$ matrices. We call
this notion of regret \emph{Procrustean swap regret} due to its
similarity with the \emph{orthogonal Procrustes problem} from linear
algebra, which (loosely) asks for the orthogonal matrix which most
closely maps one sequence of points $\bm{x}_1$ onto another sequence
of points $\bm{x}_2$ (in our setting, we intuitively want to map
$\bm{x}$ onto $-\bm{\ell}$ to minimize the loss of our benchmark). See
\cite{gower2004procrustes} for a more detailed discussion of the
Procrustes problem. Regardless, note that we can compute
$\PReg(\bm{p}, \bm{\ell})$ efficiently, since we have an efficient
membership oracle for the convex hull $\conv(O(n))$ of the set of
orthogonal matrices (specifically, an $n$-by-$n$ matrix $M$ belongs to
$\conv(O(n))$ iff all its singular values are at most $1$ in absolute
value).

In our approachability framework, we can capture this notion of regret
with the bilinear function $u(p, \ell)$ with coordinates indexed by $Q
\in O(n)$ with $u(p, \ell)_Q = \langle p, \ell \rangle - \langle Qp,
\ell \rangle$ (which is separable with respect to the negative
orthant, since for $p = -\ell/||\ell||_2$, $u(p, \ell)_Q \leq
0$). Since all the conditions of Theorem \ref{thm:efficient_alg} hold,
there is an efficient learning algorithm which incurs at most
$O(\sqrt{nT})$ Procrustean swap regret (in particular, $\diam \sP =
O(1)$, $\diam \sL = O(\sqrt{n})$, and $D_z = O(1)$).

\subsection{Converging to Bayesian correlated equilibria}
\label{sec:bayes_correlated}

Swap regret has the nice property that if in a repeated $n$-player
normal-form game, all players run a low-swap regret algorithm to
select their actions, their time-averaged strategy profile will
converge to a correlated equilibrium (indeed, this is one of the major
motivations for studying swap regret).

In repeated Bayesian games (games where each player has private
information drawn independently from some distribution each round) the
analogue of correlated equilibria is \textit{Bayesian correlated
  equilibria}. Playing a repeated Bayesian game requires a contextual
learning algorithm, which can observe the private information of the
player (the ``context'') and select an action based on
this. \cite{mansour2022} show that there is a notion of regret (that
we call $\BSReg$) such that if all learners are playing an algorithm
with low $\BSReg$, then over time they converge on average to a
Bayesian correlated equilibrium. However, while \cite{mansour2022}
provide a low $\BSReg$ algorithm, their algorithm is not provably
efficient (it requires finding the fixed point of a system of
quadratic equations); by applying our framework, we show that it is
possible to obtain a polynomial-time, low-$\BSReg$ algorithm for this
problem.

Formally, we study the following full-information contextual online
learning setting. As before, there are $K$ actions, but there are now
$C$ different contexts (``types''). Every round $t$, the adversary
specifies a loss function $\ell_t \in \sL = [0, 1]^{CK}$, where
$\ell_{t, i, c}$ represents the loss from playing action $i$ in
context $c$. Simultaneously, the learner specifies an action $p_t \in
\sP = \Delta([K])^C \subset \Rset^{CK}$ which we view as a function
mapping each context to a distribution over actions. Overall, the
learner receives expected utility $\sum_{c=1}^{C} \Pr_{\cC}[c]
\sum_{i=1}^{K} p_{t}(c)_i \ell_{t, i, c}$ this round (the learner's
context is drawn iid from the publicly known distribution $\cC$ each
round). In this formulation $\BSReg$ can be written in the form
\begin{equation}
\label{eq:bsreg_main1}
\BSReg
= \max_{\kappa, \pi_{c}}\sum_{t=1}^{T}\sum_{c=1}^{C} \Pr_{\cC}[c]
\sum_{i=1}^{K} \left(p_{t}(c)_i\ell_{t, i, c}
- p_{t}(\kappa(c))_{i}\ell_{t, \pi_{c}(i), c}\right),
\end{equation}
where the maximum is over all ``type swap functions'' $\kappa\colon [C]
\to [C]$ and $C$-tuples of ``action-deviation swap functions''
$\pi_{c}\colon [K] \to [K]$. It is straightforward to verify that
$\BSReg$ (as written in \eqref{eq:bsreg_main1}) can be written as an
$\ell_{\infty}$-approachability problem for a bilinear function $u\colon
\sP \times \sL \to \Rset^{d}$ for $d = C^C K^{KC}$. The theory
of $\ell_{\infty}$-approachability guarantees the existence of an
algorithm with $O(\sqrt{T\log d}) = O(\sqrt{T(C\log C + KC \log K)})$
regret, but this algorithm has time/space complexity $O(d)$ and is
very inefficient for even moderate values of $C$ or $K$.

Instead, in Appendix \ref{app:bayesian}, we show that $\BSReg$ can be
written in the form

$$\BSReg
= \sum_{c = 1}^{C} \max_{c' \in [C]} \sum_{i = 1}^{K} \max_{j \in [K]}
\sum_{t=1}^{T} \Pr_{\cC}[c] \left(p_{t}(c)_i\ell_{t, i, c}
- p_{t}(c')_{i}\ell_{t, j, c}\right).$$

\noindent
This allows us to evaluate $\BSReg$ in $\poly(C, K, T)$ time and apply
our pseudonorm approachability framework. Directly from Theorems
\ref{thm:poly_regret} and \ref{thm:efficient_alg}, we know that there
exists an efficient ($\poly(C, K)$ time per round) learning algorithm
that incurs at most $O(K^2C^2\sqrt{T})$ swap regret. As with swap
regret, we can tighten this bound somewhat by examining the values of
$\diam(\sX)$ and $\diam(\sY)$ and show that this algorithm actually
incurs at most $O(KC\sqrt{T})$ regret (details left to Appendix
\ref{app:bayesian}).

This is within approximately an $O(\sqrt{KC})$ factor of
optimal. Interestingly, unlike with swap regret, it is unclear if it
is possible to efficiently solve the relevant entropy maximization
problem for Bayesian swap regret (and hence achieve the optimal regret
bound). We pose this as an open question.

\begin{question}
Is it possible to efficiently (in $\poly(K, C)$ time) evaluate the
maximum entropy regularizer $\wtR(\wttheta)$ for the problem of
Bayesian swap regret?
\end{question}

\subsection{Reinforcement learning in constrained MDPs}
\label{sec:rl}

We consider episodic reinforcement learning in constrained Markov
decision processes (MDPs). Here, the agent receives a vectorial reward
(loss) in each time step and aims to find a policy that achieves a
certain minimum total reward (maximum total loss) in each dimension.
Approachability has been used to derive reinforcement learning
algorithms for constrained MDPs before
\citep{MiryoosefiBrantleyDaumeDudikSchapire2019, yu2021provably,
  MiryoosefiJin2021}, however, exclusively using $\ell_2$ geometry. As
a result, these methods aim to bound the $\ell_2$ distance to the
feasible set and the bounds scale with the $\ell_2$ norm of the reward
vector. This deviates from the more common, and perhaps more natural,
formulation for constrained MDPs studied in other works
\citep{efroni2020exploration,brantley2020constrained,ding2021provably}. Here,
each component of the loss vector is within a given range (e.g.  $[0,
  1]$) and the goal is to minimize the largest constraint violation
among all components. We will show that $\ell_\infty$-approachability
is the natural approach for this problem and yields algorithms that
avoid the $\sqrt{d}$ factor in the regret, where $d$ is the number of
constraints, that algorithms based on $\ell_2$ approachability suffer.
While the number of constraints $d$ can sometimes be small, there are
many applications where $d$ is large and a $\operatorname{poly}(d)$
dependency in the regret is undesirable, even when a computational
dependency on $d$ is manageable. For example, constraints may arise
from fairness considerations like demographic disparity that ensure
that the policy behaves similarly across all protected groups. This
would require a constraint for each pair of groups which could be very
many.

The formal problem setup is as follows. We consider an MDP $(\sX, \sA,
P, \curl*{\ell}_{t \in [T]})$ defined by a state space $\sX$, an
action set $\sA$, a transition function $P \colon \sX \times \sA
\times \sX \to [0, 1]$, where $P(x' | x, a)$ is the probability of
reaching state $x'$ when choosing action $a$ at state $x$, and
$\ell\colon \sX \times \sA \to [0, 1]^d$ the loss vector function with
$\ell = (\ell^1, \ldots, \ell^d)$.  We work with losses for
consistency with the other sections but our results also readily apply
to rewards.  To simplify the presentation, we will assume a
\emph{layered MDP} with $L + 1$ layers with $\sX = \bigcup_{l = 0}^L
\sX_l$, $\sX_i \cap \sX_j = \emptyset$ for $i \neq j$, and with $\sX_0
= \{x_0\}$ and $\sX_L = \{x_L\}$.

We define a (stochastic) policy $\pi$ as mapping from $\sA \times \sX
\to [0, 1]$ where $\pi(a | x)$ represents the probability of action
$a$ in state $x$. Given a policy $\pi$ and the transition probability
$P$, we define the \emph{occupancy measure} $\sfq^{P, \pi}$ as the
probability of visiting state-action pair $(x, a)$ when following
policy $P$ \citep{Altman1999,NeuGyorgySzepesvari2012}: $\sfq^{P,
  \pi}(x, a) = \P\bracket*{x, a \mid P, \pi}$.  We will denote by
$\Delta(P)$ the set of all occupancy measures, obtained by varying the
policy $\pi$. It is known that $\Delta(P)$ forms a polytope.

We consider the feasibility problem in CMDPs with stochastic rewards
and unknown dynamics. The loss vector $\ell\colon \sX \times \sA \to
[0, 1]^d$ is the same in all episodes, $\sL = \{ \ell \}$. Our goal is
to learn a policy $\pi \in \Pi$ such that $\langle \sfq^{P, \pi}, \ell
\rangle \leq \bdc $ for a given threshold vector $\bdc \in [0,
  L]^d$. The payoff function $u\colon \Delta(P) \times \sL \to
\Rset^d$ is defined as: $u(\sfp, \ell) = \bracket*{
\begin{smallmatrix}
\sfp \cdot  \ell^1 - c_1\\
\vdots\\
\sfp \cdot  \ell^d - c_d
\end{smallmatrix}}$.
The set $\sS = (-\infty, 0]^d$ is separable as long as there is a
policy that satisfies all constraints. Although we define the payoff
function $u$ in terms of occupancy measures $\Delta(P)$, they will be
implicit in the algorithm.

To aide the comparison of our approach to existing work in this
setting, we omit the dimensionality reduction with pseudonorms in
this application and directly work in the $d$-dimensional space. We
will analyze Algorithm~\ref{algo:algorithm-linf-approach} for
$\ell_\infty$ approachability, which can be implemented in MDPs by
adopting the following oracles from prior work
\citep{MiryoosefiBrantleyDaumeDudikSchapire2019} on CMDPs:
\begin{itemize}

\item \textbf{\textsc{BestResponse}-Oracle} For a given $\theta \in
  \Rset^d$, this oracle returns a policy $\pi$ that is
  $\epsilon_0$-optimal with respect to the reward function $(s,a)
  \mapsto -\langle \ell(s,a), \theta\rangle$.
  
\item \textbf{\textsc{Est}-Oracle} For a given policy $\pi$, this
  oracle returns a vector $z \in [0, L]^d$ such that $\|z - \langle
  \sfq^{P, \pi}, \ell \rangle\|_\infty \leq \epsilon_1$.
\end{itemize}

Consider first the case without approximation errors, $\epsilon_0 =
\epsilon_1 = 0$, for illustration. For a vector $\theta_t \in
\Delta_{d+1}$, let $\theta_t' \in \Rset^d$ be the vector that contains
only the first $d$ dimensions of $\theta_t$. When we call the
\textsc{BestResponse} oracle with vector $\theta_t'$, it returns a
policy $\pi_t$ such that its occupancy measure satisfies $\langle
\sfq^{P, \pi_t}, - \theta_t' \cdot \ell \rangle \geq \max_{\sfq^\star
  \in \Delta(P)} \langle \sfq^\star, - \theta_t' \cdot \ell \rangle$.
We can use this to show:
\begin{align*}
    \theta_t \cdot [u(\sfq^{P, \pi_t}, \ell), ~0]
    & = \theta_t' \cdot u(\sfq^{P, \pi_t}, \ell)
    = \langle \sfq^{P, \pi_t}, \theta_t' \cdot \ell \rangle
    - c \leq \langle \sfq^\star, \theta_t' \cdot \ell \rangle - c \leq 0
\end{align*}
where $\sfq^\star$ is the occupancy measure of a policy that satisfies
all constraints.

Thus $\sfq^{P, \pi_t}$ is a valid choice for $p_t$ in Line~3 of
Algorithm~\ref{algo:algorithm-linf-approach}. Passing the policy
$\pi_t$ to the \textsc{Est} yields $\hat z_t = u(\sfq^{P, \pi_t},
\ell) + c$; this is enough to compute
\begin{align*}
    y_t &= - [\hat z_t - c,~ 0] = -  [u(p_t, \ell), ~0]
\end{align*}
in Line~5 of
Algorithm~\ref{algo:algorithm-linf-approach}\footnote{Since FTRL with
negative entropy regularizer operates on the simplex with $\|
\theta_{t}\| = 1$, we pad the inputs with a zero dimension obtain an
OLO algorithm on the interior of $\Delta_{d}$.}. Finally, we obtain
$\theta_{t+1} \in \mathbb R^{d+1}$ by passing $y_1, y_2, \dots, y_t$ to
a negative entropy FTRL algorithm as the OLO algorithm $\cF$ (Line~6
of Alg.~\ref{algo:algorithm-linf-approach}).

Using similar steps as for Theorem~\ref{thm:linf_approachability} and
relying on the regret bound for $\cF$ in
Lemma~\ref{lem:negent_regularizer}, we can show the following
guarantee:

\begin{restatable}{proposition}{appropoinfty}
\label{prop:constrainedmdp}
Consider a constrained episodic MDP with horizon $L$, fixed loss
vectors $\ell$ with $\ell(s,a) \in [0, 1]^d$ for all state-action
pairs $(s,a) \in \sX \times \sA$ and a constraint threshold vector $c
\in [0, L]^d$.  Assume that there exists a feasible policy that
satisfies all constraints and let $\bar \pi$ be the mixture policy of
$\pi_1, \dots, \pi_T$ generated by the approach described above. Then
the maximum constraint violation of $\bar \pi$ satisfies
\begin{align*}
    \max_{i \in [d]} \{\sfq^{P, \bar \pi} \cdot \ell^i - c_i \}
    = d_{\infty}\left(\frac{1}{T}\sum_{t=1}^T u(\sfq^{P, \pi_t}, \ell), \sS \right)
    \leq O\left(L \sqrt{\frac{\log(d)}{T}}\right) + \epsilon_0 + 2\epsilon_1.
\end{align*}
\end{restatable}
Applying the results from prior work
\citep{MiryoosefiBrantleyDaumeDudikSchapire2019} based on $\ell_2$
approachability would yield a bound of $(L\sqrt{d} + \epsilon_1)
T^{-1/2} + \epsilon_0 + 2\epsilon_1$ in our setting with an additional
$\sqrt{d}$ and $\epsilon_1$ factor in front of $T^{-1/2}$.  For the
sake of exposition, we illustrated the benefit of $\ell_\infty$
approachability using the oracles adopted by
\citet{MiryoosefiBrantleyDaumeDudikSchapire2019}, but our approach can
also be applied with similar advantages to other works that make
oracle calls explicit \citep{yu2021provably, MiryoosefiJin2021}.

\section{Conclusion}

  We presented a new algorithmic framework for
  $\ell_\infty$-approachability, which we argued is the most suitable
  notion of approachability for a variety of applications such as
  regret minimization. Our algorithms leverage a key dimensionality
  reduction and a reduction to an online linear optimization. These
  ideas can be similarly used to derive useful algorithms for
  approachability with alternative distance metrics. In fact, as
  already pointed out, some of our algorithms can be similarly
  viewed as a reduction to an online linear optimization of an
  equivalent group-norm approachability.

\bibliographystyle{plainnat}
\bibliography{pseudonorm_approachability.bib}

\newpage
\appendix

\renewcommand{\contentsname}{Contents of Appendix}
\tableofcontents
\addtocontents{toc}{\protect\setcounter{tocdepth}{3}} 
\clearpage

\section{Omitted proofs}
\label{app:omitted}

\subsection{Proof of Theorem~\ref{thm:linf_approachability}}

\begin{proof}[Proof of Theorem~\ref{thm:linf_approachability}]
Similar results appear in e.g. \cite{Kwon2021}. For completeness, we
include a proof here.

We will need the following fact: for any $x \in \Rset^d$,
$d_{\infty}(x, (\infty, 0]^d) = \sup_{\theta \in \ODelta_d}
  \tri{\theta, x}$. This is easy to verify (both sides are equal to
  $\max(\max_{i \in [d]} x_i, 0)$), but is also a consequence of
  Fenchel duality (see e.g. Appendix \ref{app:fenchel-duality} and the
  proof of Theorem \ref{th:general-function}).

Armed with this fact, note that
\begin{eqnarray*}
d_\infty\left(\frac{1}{T}\sum_{t=1}^{T}u(p_{t}, \ell_t), (-\infty, 0]^d \right) &=& \sup_{\theta \in \ODelta_d} \left\langle \theta, \frac{1}{T}\sum_{t=1}^{T}u(p_{t}, \ell_t)\right\rangle\\
&=& -\inf_{\theta \in \sX} \left\langle \theta, \frac{1}{T}\sum_{t=1}^{T}\left(-u(p_{t}, \ell_t)\right)\right\rangle\\
&=& \frac{1}{T}\Reg(\cF) - \frac{1}{T}\sum_{t=1}^{T}\left\langle \theta_t, -u(p_{t}, \ell_t)\right\rangle \\
&\leq& \frac{1}{T}\Reg(\cF).
\end{eqnarray*}

\end{proof}

\subsection{Proof of Corollary~\ref{cor:quasidistance_olo}}

\begin{proof}[Proof of Corollary~\ref{cor:quasidistance_olo}]
Note that if $\theta \not \in \sC^{\circ}$, then there exists a $z
\in \sC$ such that $\langle \theta, z \rangle > 0$. Therefore, if
$\theta \not \in \sC^{\circ}$, then the supremum $\sup_{s \in \sC}
\theta \cdot s = \infty$ (we can take $s$ to be a large multiple of
$z$). On the other hand, if $\theta \in \sC^{\circ}$, then the
supremum $\sup_{s \in \sC}\theta \cdot s = 0$ (taking $s = 0$). It
follows that the supremum $\sup_{\theta \in \sT^*_f} \curl*{\theta
  \cdot z - \sup_{s \in \sC} \theta \cdot s}$ in Theorem
\ref{th:general-function} must be achieved for a $\theta \in
\sC^{\circ}$ (we know that $d_f(z, \sC)$ is not infinite since $0 \in
\sC$, so $d_f(z, \sC) \leq f(z)$). For such $\theta$, the $\sup_{s \in
  \sS} \theta \cdot s$ term vanishes, and we are left with the
statement of this corollary.
\end{proof}

\subsection{Proof of Theorem~\ref{thm:dimension_reduction}}

\begin{proof}[Proof of Theorem~\ref{thm:dimension_reduction}]
Note that for any $\rho \geq 0$ we have the following chain of equivalences:

\begin{eqnarray}
& & d_{\infty}\left(\frac{1}{T}\sum_{t=1}^{T} u(p_t, \ell_t), (-\infty, 0]^{d}\right) \leq \rho\nonumber\\
&\Leftrightarrow& \frac{1}{T}\sum_{t=1}^{T} u_i(p_t, \ell_t) \leq \rho \; \forall i \in [d]\nonumber \\
&\Leftrightarrow& \frac{1}{T}\sum_{t=1}^{T} \langle \wtu(p_t, \ell_t), v_i \rangle \leq \rho \; \forall i \in [d]\nonumber \\
&\Leftrightarrow& \left\langle\frac{1}{T}\sum_{t=1}^{T}  \wtu(p_t, \ell_t), v_i \right\rangle \leq \rho \; \forall i \in [d]\nonumber \\
&\Leftrightarrow& f\left(\frac{1}{T}\sum_{t=1}^{T}  \wtu(p_t, \ell_t)\right) \leq \rho. \label{eq:lasteq}
\end{eqnarray}

We will now prove that this final inequality on the norm $f$ \eqref{eq:lasteq} implies the desired inequality on $d_f$:

\begin{equation}\label{eq:df_ineq}
d_{f}\left(\frac{1}{T}\sum_{t=1}^{T} \wtu(p_t, \ell_t), \sC\right) \leq \rho.
\end{equation}

Note that since $0 \in \sC$, \eqref{eq:df_ineq} directly implies \ref{eq:lasteq}. To show that \eqref{eq:lasteq} implies \eqref{eq:df_ineq}, we must show that

$$f\left(\left(\frac{1}{T}\sum_{t=1}^{T}  \wtu(p_t, \ell_t)\right) - c\right) \leq \rho.$$

\noindent
for any $c \in \sC$. This in turn is equivalent to proving that:

$$\left\langle\left(\frac{1}{T}\sum_{t=1}^{T}  \wtu(p_t, \ell_t)\right) - c, v_i \right\rangle \leq \rho \; \forall i \in [d]\nonumber.$$

\noindent
But since $c \in \sC$, $\langle c, v_i \rangle \geq 0$ for all $i \in [d]$. Therefore, as long as 

$$\left\langle\frac{1}{T}\sum_{t=1}^{T}  \wtu(p_t, \ell_t), v_i \right\rangle \leq \rho \; \forall i \in [d],$$

\noindent
(which is true by \eqref{eq:lasteq}), \eqref{eq:df_ineq} will be true as well, as desired.

\end{proof}

\subsection{Proof of Lemma~\ref{lemma:DualSet}}

\begin{proof}[Proof of Lemma~\ref{lemma:DualSet}]
  Let $\sH$ denote $\conv \curl*{v_1, \ldots, v_d}$. If $\wttheta$ is
  in $\sH$, then we can write $\wttheta = \sum_{i = 1}^d \alpha_i v_i$
  for some $\alpha_i \geq 0$, $i \in [d]$, with $\sum_{i = 1}^d
  \alpha_i = 1$. Thus, for any $x \in \Rset^{mn}$, we have
  \[
  \tri{\wttheta, x} = \sum_{i = 1}^d \alpha_i \tri{v_i, x}
  \leq \sum_{i = 1}^d \alpha_i \max_{j \in [d]} \tri{v_j, x}
  = \max_{j \in [d]} \tri{v_j, x},
  \]
  which implies that $\wttheta$ is in $\sT^*_{f}$.
  Conversely, if $\wttheta$ is not in $\sH$, since $\sH$ is a non-empty
  closed convex set, $\wttheta$ can be separated from $\sH$, that is
  there exists $x \in \Rset^d$ such that
  \[
\tri{\wttheta, x} > \sup_{v \in \sH} \tri{v, x} \geq \max_{i \in [d]} \tri{v_i, x},
\]
which implies that $\wttheta$ is not in $\sT^*_{f}$. This
completes the proof.
\end{proof}

\subsection{Proof of Lemma~\ref{lemma:low-dim-separability}}

\begin{proof}[Proof of Lemma~\ref{lemma:low-dim-separability}]
We begin with the first statement. Note that since the negative
orthant $(-\infty, 0]^d$ is separable with respect to $u$, for any
  $\ell \in \sL$, there exists a $p \in \sP$ such that $u_i(p, \ell)
  \leq 0$ for all $i \in [d]$. Now, recall that we can write $u_i(p,
  \ell) = \tri{\wtu(p, \ell), v_i}$, so it follows that $\tri{\wtu(p,
    \ell), v_i} \leq 0$ for all $i \in [d]$. This implies $\wtu(p,
  \ell) \in \sC$, as desired.

We next prove the second statement. Note that since $\wttheta \in
\sT_{f}^* \subseteq \sC^{\circ}$, this implies that if $\wtu(p, \ell)
\in \sC$, then $\tri{\wttheta, \wtu(p, \ell)} \leq 0$. By the first
statement, this means that for any $\ell \in \sL$, there exists a $p
\in \sP$ such that $\tri{\wttheta, \wtu(p, \ell)} \leq 0$. By the
minimax theorem (since $\sP$ and $\sL$ are both convex sets and
$\tri{\wttheta, \wtu(p, \ell)}$ is a bilinear function of $p$ and
$\ell$), this implies that there exists a $p \in \sP$ such that for
all $\ell \in \sL$, $\tri{\wttheta, \wtu(p, \ell)} \leq 0$, as
desired.
\end{proof}

\subsection{Proof of Lemma~\ref{lem:dual_extension}}

\begin{proof}[Proof of Lemma~\ref{lem:dual_extension}]
Note that by the definition of $\sT^{*}_{f}(\sZ)$, it must be the case
that $\sup_{\theta \in \sT^{*}_f(\sZ)} \tri{\theta, z} \leq f(z)$
(this is one of the constraints defining $\sT^{*}_f(\sZ)$). Since
$\sT^{*}_{f} \subseteq \sT^{*}_{f}(\sZ)$, if we show that
$\sup_{\theta \in \sT^{*}_f} \tri{\theta, z} = f(z)$ this proves the
theorem.
But this follows directly from Theorem \ref{th:general-function}
applied to the closed convex set $\sS = \{0\}$.
\end{proof}

\subsection{Proof of Lemma~\ref{lemma:bound_extension}}

\begin{proof}[Proof of Lemma~\ref{lemma:bound_extension}]
Consider the element

$$\wttheta' = \argmin_{\wttheta \in \sT_{f}^*} - \langle \wttheta, z \rangle$$

\noindent
which minimizes the unregularized objective over the smaller set
$\sT_{f}^*$. By Lemma \ref{lem:dual_extension}, we must have that
$\langle \wttheta', z\rangle \geq \langle \wttheta_{opt},
z\rangle$. But also, by definition of $\wttheta_{opt}$, we must have
that $\eta R(\wttheta_{opt}) - \langle \wttheta_{opt}, z \rangle \leq
\eta R(\wttheta') - \langle \wttheta', z \rangle$. It follows that
$R(\wttheta_{opt}) \leq R(\wttheta')$, and thus $||\wttheta_{opt}||
\leq ||\wttheta'|| \leq \diam(\sT_{f}^*)$.
\end{proof}

\subsection{Proof of Lemma~\ref{lem:z_membership}}

\begin{proof}[Proof of Lemma~\ref{lem:z_membership}]
Extend the domain of $\wtu$ as a function to $\Rset^{n} \times
\Rset^{m}$, and note that there is a unique way to write $z =
\sum_{k=1}^{m} \wtu(z_{k}, e_k)$, where for each $k \in [m]$, $z_{k}$
is an element of $\Rset^{n}$ and $e_k$ is the $k$th unit vector in
$\Rset^{m}$. We first claim $z \in \sZ = \cone(\wtu(\sP, \sL))$ iff
each $z_{k} \in \cone(\sP)$.

To see this, first note that since $\sL$ is orthant-generating, there
exist a sequence of $\lambda_k$ such that $\lambda_ke_k \in \sL$ for
each $k \in [m]$. Now, if each $z_{k} \in \cone(\sP)$, then
$\wtu(z_{k}, e_k) \in \cone(\wtu(\sP, \sL))$ (since $e_k \in
\cone(\sL)$), so $z \in \sZ$. Conversely, if $z \in \sZ$, then we can
write $z = \sum \alpha_{r=1}^{R} \wtu(p_r, \ell_r)$ for some $R > 0$
and $\alpha_r \geq 0$. Expanding each $\ell_{r}$ in the
$\{\lambda_ke_k\}$ basis, we find that each $z_k$ must be a positive
linear combination of the values $p_r$ and therefore $z_{k} \in
\cone(\sP)$.

Therefore, to check whether $z \in \sZ$, it suffices to check whether
each component $z_k$ of $z$ belongs to $\cone(\sP)$. This is possible
to do efficiently given an efficient separation oracle for $\sP$ (we
can write the convex program $\beta z_k \in \sP$ for $\beta >
0$). Finally, if each $z_k \in \cone(\sP)$ we can also recover a value
$\beta_k$ for each $k$ such that $\beta_kz_k \in \sP$ (via the same
convex program). This allows us to explicitly write $z =
\sum_{k=1}^{m}\alpha_k \wtu(p_k, \ell_k)$ with $p_k = \beta_kz_k$,
$\ell_k = \lambda_ke_k$, and $\alpha_k = 1/(\beta_k\lambda_k)$.
\end{proof}

\subsection{Proof of Lemma~\ref{lem:x_membership}}

\begin{proof}[Proof of Lemma~\ref{lem:x_membership}]
Checking for membership in the ball of radius $\rho$ is
straightforward, so it suffices to exhibit a membership oracle for the
set $\sT_{f}^{*}(\sZ)$. Fix a $\wttheta$ and consider the convex
function $h(z) = \max_{i \in [d]}\tri{v_i, z} - \tri{\wttheta,
  z}$. Note that by the definition of $\wttheta \in \sT_{f}^{*}(\sZ)$
iff $\min_{z \in \sZ}h(z) \leq 0$, so it suffices to compute the
minimum of $h(z)$ over the convex set $\sZ$.

As mentioned, to do this it suffices to exhibit a membership oracle
for $\sZ$ and an evaluation oracle for $h(z)$. But now, Lemma
\ref{lem:z_membership} provides a membership oracle for $\sZ$ and
Corollary \ref{cor:eval_max} allows us to efficiently evaluate $h(z)$
for $z \in \sZ$.
\end{proof}

\subsection{Proof of Proposition~\ref{prop:constrainedmdp}}

\begin{proof}[Proof of Proposition~\ref{prop:constrainedmdp}]
We define $y_t = - [ u(p_t, \ell), ~0]$ the noiseless and $\tilde y_t = -[\hat z_t - c, 0]$ the actual loss passed to the OLO algorithm.
We bound
\begin{align*}
    \max_{i \in [k]} \{\sfq^{P, \bar \pi} \cdot \ell - c \}
    &= d_{\infty}\left(\frac{1}{T}\sum_{t=1}^T u(\sfq^{P, \pi_t}, \ell), \sS \right)\\
    &= \max_{\theta \in \Delta_{d+1}} \theta \cdot \begin{pmatrix}
    \frac{1}{T}\sum_{t=1}^T u(\sfq^{P, \pi_t}, \ell)\\
    0
  \end{pmatrix} \\ 
  &= \max_{\theta \in \Delta_{d+1}} \frac{1}{T}\sum_{t=1}^T \theta \cdot \begin{pmatrix}
     u(\sfq^{P, \pi_t}, \ell)\\
    0
  \end{pmatrix}\\
  &\leq \max_{\theta \in \Delta_{d+1}} \left\{-\frac{1}{T}\sum_{t=1}^T \langle \theta, \tilde y_t \rangle + \epsilon_1 \right\} \tag{def. $\tilde y_t$ and \textsc{Est} oracle}\\
  &= - \min_{\theta \in \Delta_{d+1}} \left\{\frac{1}{T}\sum_{t=1}^T \langle \theta, \tilde y_t \rangle  \right\} + \epsilon_1 
  \\
  &\leq \frac{\Reg(\cF)}{T} - \frac{1}{T} \sum_{t = 1}^T \langle \theta_t, \tilde y_t \rangle + \epsilon_1 
    \\
  &\leq \frac{\Reg(\cF)}{T} - \frac{1}{T} \sum_{t = 1}^T \langle \theta_t, y_t \rangle + \epsilon_1 \tag{\textsc{Est} oracle}
  \\
    &\leq \frac{\Reg(\cF)}{T}  + \frac{1}{T} \sum_{t = 1}^T \left\langle \theta_t, \begin{pmatrix}
     u(\sfq^{P, \pi_t}, \ell)\\
    0
  \end{pmatrix}\right\rangle + 2\epsilon_1 \tag{definition $y_t$}\\
      &\leq \frac{\Reg(\cF)}{T}  + \frac{1}{T} \sum_{t = 1}^T \left\langle \theta_t \cdot \begin{pmatrix}
     u(\sfq^{P, \pi_t}, \ell)\\
    0
  \end{pmatrix}\right\rangle + 2\epsilon_1 \tag{definition $y_t$}\\
  &\leq \frac{\Reg(\cF)}{T}  + \epsilon_0 + 2\epsilon_1 \tag{\textsc{BestResponse} oracle}\\
  &\leq 2 L \sqrt{\frac{\log(d+1)}{T}} + \epsilon_0 + 2\epsilon_1 \tag{Lemma~\ref{lem:negent_regularizer}}
\end{align*}
\end{proof}

\section{General theorems from convex optimization}
\label{app:fenchel-duality}

We will use the following Fenchel duality theorem
\citep{BorweinZhu2005,Rockafellar1970}, see also \citep{MohriRostamizadehTalwalkar2018}.

\begin{theorem}[Fenchel duality]
\label{th:fenchel-duality}
Let $\sX$ and $\sY$ be Banach spaces, $f\colon \sX \to \Rset \cup
\{+\infty\}$ and $g \colon \sY \to \Rset \cup \{+\infty\}$
convex functions and $A \colon \sX \to \sY$ a bounded linear map.
Assume that $f$, $g$ and $A$ satisfy one of the following conditions:
\begin{itemize}

\item $f$ and $g$ are lower semi-continuous and $0 \in \core(\dom(g) -
  A \dom(f))$;

\item $A \dom(f) \cap \cont(g) \neq \emptyset$;
  
\end{itemize}
then $p = d$ for the dual optimization problems
\begin{align*}
  p & = \inf_{x \in \sX} \curl*{f(x) + g(Ax)}\\
  d & = \sup_{x^* \in \sY^*} \curl*{-f^*(A^*x^*) - g^*(-x^*)}
\end{align*}
and the supremum in the second problem is attained if finite.

\end{theorem}

When constructing efficient algorithms, we will need a efficient
method for minimizing a convex function over a convex set given only a
membership oracle to the set and an evaluation oracle to the
function. This is provided by the following lemma.

\begin{lemma}
\label{lemma:convex-opt}
Let $K$ be a bounded convex subset of $\Rset^{d}$ and $f:K \to
[0, 1]$ a convex function over $K$. Then given access to a membership
oracle for $K$ (along with an interior point $x_0 \in K$ satisfying
$B(x_0, r) \subseteq K \subseteq B(x_0, R)$ for some given radii $r, R
> 0$) and an evaluation oracle for $f$, there exists an algorithm
which computes the minimum value over $f$ (to within precision $\eps$)
using $\poly(d, \log(R/r), \log(1/\eps))$ time and queries to these
oracles.
\end{lemma}
\begin{proof}
See \cite{lee2018efficient}.
\end{proof}

\section{Bayesian correlated equilibria}
\label{app:bayesian}

Here we provide another application of the $\ell_{\infty}$
approachability framework, to the problem of constructing learning
algorithms that converge to correlated equilibria in Bayesian
games. Correlated equilibria in Bayesian games (alternatively, ``games
with incomplete information'') are well-studied throughout the
economics and game theory literature; see e.g. \citep{forges1993five,
  forges2006correlated, bergemann2016bayes}. Unlike ordinary
correlated equilibria, which are also well-studied from a learning
perspective, relatively little is known about algorithms that converge
to correlated equilibria in Bayesian games.  \cite{hartline2015no}
study no-regret learning in Bayesian games, showing that no-regret
algorithms converge to a Bayesian coarse correlated equilibrium. More
recently, \cite{mansour2022} introduce a notion of Bayesian swap
regret with the property that learners with sublinear Bayesian swap
regret converge to correlated equilibria in Bayesian
games. \cite{mansour2022} construct a learning algorithm that achieves
low Bayesian swap regret, albeit not a provably efficient one. In this
section, we apply our approachability framework to develop the first
efficient low-regret algorithm for Bayesian swap regret.

We begin with some preliminaries about standard (non-Bayesian)
normal-form games. In a normal form game $G$ with $N$ players, each
player $n$ must choose a mixed action $\bx_n \in \Delta([K])$ (for
simplicity we will assume each player has the same number of pure
actions). We call the collection of mixed strategies $\bx = (\bx_1,
\bx_2, \dots, \bx_N)$ played by all players a \textit{strategy
  profile}. We will let $U_i(\bx)$ denote the utility of player $n$
under strategy profile $\bx$ (and insist that $U_i$ is linear in each
player's strategy).

Given a function $\pi\colon [K] \to [K]$ and a mixed action
$\bx_n \in \Delta([K])$, let $\pi(\bx_n)$ be the mixed action in
$\Delta([K])$ formed by sampling an action from $\bx_n$ and then
applying the function $\pi$ (i.e., $\pi(\bx_n)_i = \sum_{i' | \pi(i')
  = i} \bx_{n, i'}$). A \textit{correlated equilibrium} of $G$ is a
distribution $\cD$ over strategy profiles $\bx$ such that for any
player $n$ and function $\pi\colon [K] \to [K]$, it is the
case that

$$\E_{\bx \sim \cD}[U_{n}(\bx')] \leq \E_{\bx \sim \cD}[U_{n}(\bx)],$$ 
\noindent
where $\bx' = (\bx_1, \bx_2, \dots, \pi(\bx_n), \dots,
\bx_N)$. Similarly, an $\eps$-correlated equilibrium of $G$ is a
distribution $\cD$ with the property that

$$\E_{\bx \sim \cD}[U_{n}(\bx')] \leq \E_{\bx \sim \cD}[U_{n}(\bx)] + \eps,$$

\noindent
for any $n \in [N]$ and $\pi\colon [K] \to [K]$. Correlated
equilibria have the following natural interpretation: a mediator
samples a strategy profile $\bx$ from $\cD$ and tells to each player
$n$ a pure action $a_n$ randomly sampled from $\bx_n$. If each player
is incentivized to play the action $a_n$ they are told, then $\cD$ is
a correlated equilibrium.

We define Bayesian games similarly to standard games, with the
modification that now each player $n$ also has some private
information $c_n \in [C]$ drawn from some public distribution
$\cC_n$. We call the vector of realized types $\bdc = (c_1, c_2,
\dots, c_N)$ the \textit{type profile} of the players (drawn randomly
from $\cC = \cC_1 \times \dots \times \cC_N$), and now let the utility
$U_n(\bx, \bdc)$ of player $n$ depend on both the strategy profile and
type profile of the players. Note that we can alternately think of the
strategy of player $n$ as a function $\bx_n\colon [C] \to
\Delta([K])$ mapping contexts to mixed actions; in this case we can
again treat the expected utility for player $n$ (with expectation
taken over the random type profile) as a multilinear function
$U_{n}(\bx)$ function of the strategy profile $\bx = (\bx_1, \dots,
\bx_N)$.

As with regular correlated equilibria, we can motivate the definition
of Bayesian correlated equilibria via the introduction of a
mediator. In the Bayesian case, all players begin by revealing their
private types to the mediator, and the mediator observes a type
profile $\bdc$. The mediator then samples an joint action profile
$\bx$ from a distribution $\cD(\bdc)$ that \textit{depends} on the
observed type profile. Finally, for each player $n$, the mediator
samples a pure action $a_n \sim \bx_n$ from the mixed strategy for
$n$, and relays $a_n$ to player $n$ (which they should follow). In
order for this to be a valid correlated equilibrium, the following
incentive compatibility constraints must be met:

\begin{itemize}
    \item Players must have no incentive to deviate from the strategy
      $\bx_n$ relayed to them. As in correlated equilibria, this
      includes deviations of the form ``if I am told to play action
      $i$, I will instead play action $j$''.

    \item Players also must have no incentive to misreport their type
      (thus affecting the distribution $\cD(\bdc)$ over joint strategy
      profiles).

    \item Moreover, no combination of the above two deviations should
      result in improved utility for a player.
\end{itemize}

Formally, we define a \textit{Bayesian correlated equilibrium} for a
Bayesian game as follows. The distributions $\cD(\bdc)$ form a
Bayesian correlated equilibrium if, for any player $n \in [N]$, any
``type deviation'' $\kappa\colon [C] \to[C]$, and any collection of
``action deviations'' $\pi_{c}\colon [K] \to [K]$ (for each $c \in
[C]$),

$$\E_{\bx \sim \cD(c), \bdc \sim \cC}[U_{n}(\bx')]
\leq \E_{\bx \sim \cD(c), \bdc \sim \cC}[U_{n}(\bx)]$$

\noindent
where $\bx'$ is derived from $\bx$ by deviating from $\bx_n$ in the
following way: when the player $n$ has type $c$, they first report
type $c' = \kappa(c)$ to the mediator; if the mediator then tells them
to play action $a$, they instead play $\pi_{c}(a)$. No such deviation
should improve the utility of an agent in a Bayesian correlated
equilibrium. Likewise, an $\eps$-Bayesian correlated equilibrium is a
collection of distributions $\cD(\bdc)$ where no deviation increases
the utility of a player by more than $\eps$.

In order to play a Bayesian game, a learning algorithm must be
contextual (using the agent's private information to decide what
action to play). We study the following setting of full-information
contextual online learning. As before, there are $K$ actions, but
there are now $C$ different contexts (/types). Every round $t$, the
adversary specifies a loss function $\ell_t \in \cL = [0, 1]^{CK}$,
where $\ell_{t, i, c}$ represents the loss from playing action $i$ in
context $c$. Simultaneously, the learner specifies an action $p_t \colon
[C] \to \Delta([K])$ mapping each context to a distribution
over actions. Overall, the learner receives expected utility
$\sum_{c=1}^{C} \Pr_{\cC}[c] \sum_{i=1}^{K} p_{t}(c)_i \ell_{t, i,
  c}$ this round (here $\cC$ is a distribution over contexts; we
assume that the learner's context is drawn i.i.d.\ from $\cC$ each round,
and that this distribution $\cC$ is publicly known).

Motivated by the deviations considered in Bayesian correlated
equilibria, we can define the following notion of swap regret in
Bayesian games (``Bayesian swap regret''):

\begin{equation}\label{eq:bs_regret}
\BSReg = \max_{\kappa, \pi_{c}}\sum_{t=1}^{T}\sum_{c=1}^{C} \Pr_{\cC}[c] \sum_{i=1}^{K} \left(p_{t}(c)_i\ell_{t, i, c}  - p_{t}(\kappa(c))_{i}\ell_{t, \pi_{c}(i), c}\right).
\end{equation}

\begin{lemma}
Let $\cA$ be an algorithm with $\BSReg(\cA) = o(T)$. Assume each
player $i$ in a repeated Bayesian game $G$ (over $T$ rounds) runs a
copy of algorithm $\cA$, and let $p_{t} = (p^{(1)}_t, p^{(2)}_t,
\dots, p^{(N)}_t)$ be the strategy profile at time $t$. Then the
time-averaged strategy profile $\cD(\bdc)\colon [C]^N \to
\Delta([K]^{N})$, defined by sampling a $t$ uniformly at random from
$[T]$ and returning $p_{t}$ is an $\eps$-Bayesian correlated
equilibrium with $\eps = o(1)$.
\end{lemma}
\begin{proof}
We will show that there exists no deviation for player $n$ which
increases their utility by more than $\BSReg(\cA)/T = o(1)$.

Fix a type deviation $\kappa\colon [C] \to [C]$ and set of
action deviations $\pi_{c}\colon [K] \to [K]$. This collection of
deviations transforms an arbitrary strategy $p \in \Delta([K])^{C}$
into the strategy $p'$ satisfying $p'(c)_{i} = \sum_{i' \mid
  \pi_{c}(i') = i} p(\kappa(c))_{i'}$.

For each $t \in [T]$, with probability $1/T$ the mediator will return
the strategy profiles $p_{t} = (p^{(1)}_t, p^{(2)}_t, \dots,
p^{(N)}_t)$. Now, since $U_{n}$ is multilinear in the strategies of
each player, there exists some vector $\ell_{t}$ such that the utility
of player $n$ if they defect to a some strategy $p'$ is given by the
inner product $\langle p, \ell_t \rangle$.

In particular, conditioned on the mediator returning $p_t$, the
difference in utility for player $n$ between playing $p^{(n)}$ and the
strategy $p'$ formed by applying the above deviations to $p^{(n)}$ is
exactly

$$\sum_{c=1}^{C} \Pr_{\cC}[c] \sum_{i=1}^{K} \left(p^{(n)}_{t}(c)_i\ell_{t, i, c}  - p^{(n)}_{t}(\kappa(c))_{i}\ell_{t, \pi_{c}(i), c}\right).$$

Taking expectations over all $t$, we have that the expected difference
in utility by deviating is

$$\frac{1}{T}\sum_{t=1}^{T}\sum_{c=1}^{C} \Pr_{\cC}[c] \sum_{i=1}^{K} \left(p^{(n)}_{t}(c)_i\ell_{t, i, c}  - p^{(n)}_{t}(\kappa(c))_{i}\ell_{t, \pi_{c}(i), c}\right).$$

But since player $n$ selected their strategies by playing $\cA$, this
is at most $\BSReg(\cA)/T = o(1)$, as desired.
\end{proof}

It is possible to phrase \eqref{eq:bs_regret} in the language of
$\ell_{\infty}$-approachability by considering the $C^C \cdot K^{KC}$
dimensional vectorial payoff $u(p, \ell)$ given by:

$$u(p, \ell)_{\kappa, \pi_{c}} = \sum_{c=1}^{C} \Pr_{\cC}[c]
\sum_{i=1}^{K} \left(p(c)_i\ell_{i, c}  - p(\kappa(c))_{i}\ell_{\pi_{c}(i), c}\right).$$

A straightforward computation shows that the negative orthant is
separable with respect to $u$.

\begin{lemma}
The set $\sS = (-\infty, 0]^d$ is separable with respect to the vectorial payoff $u$. 
\end{lemma}
\begin{proof}
Fix an $\ell \in [0, 1]^{CK}$. Then note that if we let $p(c) =
e_i$, where $i = \argmin_{j \in [K]} \ell_{t, j, c}$ (i.e., $p(c)$
is entirely supported on the best fixed action to play in context
$c$), it follows that $p(c)_i\ell_{t, i, c} \leq
p(c')_{i}\ell_{t, j, c}$ for all $c' \in [C]$ and $j \in [K]$, and
therefore that $u(p, \ell) \in \sS$.
\end{proof}

This in turn leads (via Theorem \ref{thm:linf_approachability}) to a
low-regret (albeit computationally inefficient) algorithm for Bayesian
swap regret. Instead, as in the case of swap regret, we will apply our
pseudonorm approachability framework. First, we will show that we can
rewrite \eqref{eq:bs_regret} in such a way that allows us to easily
evaluate $\BSReg$.

\begin{lemma}
We have that

\begin{equation}\label{eq:bs_regret2}
\BSReg = \sum_{c = 1}^{C} \max_{c' \in [C]} \sum_{i = 1}^{K} \max_{j \in [K]} \sum_{t=1}^{T} \Pr_{\cC}[c] \left(p_{t}(c)_i\ell_{t, i, c}  - p_{t}(c')_{i}\ell_{t, j, c}\right).
\end{equation}
\end{lemma}
\begin{proof}
From \eqref{eq:bs_regret}, we have that

\begin{eqnarray*}
\BSReg &=& \max_{\kappa, \pi_{c}}\sum_{t=1}^{T}\sum_{c=1}^{C} \Pr_{\cC}[c] \sum_{i=1}^{K} \left(p_{t}(c)_i\ell_{t, i, c}  - p_{t}(\kappa(c))_{i}\ell_{t, \pi_{c}(i), c}\right)\\
&=& \max_{\kappa, \pi_{c}}\sum_{c=1}^{C}  \sum_{i=1}^{K} \sum_{t=1}^{T}\Pr_{\cC}[c]\left(p_{t}(c)_i\ell_{t, i, c}  - p_{t}(\kappa(c))_{i}\ell_{t, \pi_{c}(i), c}\right)\\
&=& \max_{\kappa:[C]\to[C]}\sum_{c=1}^{C}  \max_{\pi_{c}\colon [K]\to[K]} \sum_{i=1}^{K} \sum_{t=1}^{T}\Pr_{\cC}[c]\left(p_{t}(c)_i\ell_{t, i, c}  - p_{t}(\kappa(c))_{i}\ell_{t, \pi_{c}(i), c}\right)\\
&=& \sum_{c=1}^{C} \max_{c' \in [C]} \sum_{i=1}^{K} \max_{j \in [K]} \sum_{t=1}^{T}\Pr_{\cC}[c]\left(p_{t}(c)_i\ell_{t, i, c}  - p_{t}(c')_{i}\ell_{t, j, c}\right).
\end{eqnarray*}
\end{proof}

Note that \eqref{eq:bs_regret2} allows us to efficiently (in $\poly(K,
C, T)$ time) evaluate $\BSReg$. As mentioned in the main text,
directly from Theorems \ref{thm:poly_regret} and
\ref{thm:efficient_alg} this gives us an efficient ($\poly(C, K)$ time
per round) learning algorithm that incurs at most $O(K^2C^2\sqrt{T})$
swap regret. We will now examine the values of $\diam(\sP)$,
$\diam(\sL)$ and $D_z$ and show that this algorithm actually incurs at
most $O(KC\sqrt{T})$ regret.

First, note that since $\sP = \Delta([K])^C$, elements $p \in \sP$ can
be thought of as $C$ $K$-tuples of positive numbers that add to
$1$. Each such $K$-tuple has squared distance at most $1$, so
$\diam(\sP) \leq \sqrt{C}$. Second, since $\sL = [0, 1]^{KC}$,
$\diam(\sL) = \sqrt{KC}$. Finally, the $KC$ coefficients of each
$z_{\kappa, \pi_c}$ consist of $2K$ copies of the distribution
$\Pr[c]$; this has $\ell_{2}$ norm at most $\sqrt{2K}$, so $D_{z} =
O(\sqrt{2K})$. Combining these three quantities according to Theorem
\ref{thm:poly_regret}, we obtain the following corollary.

\begin{corollary}
\sloppy{There exists an efficient contextual learning algorithm with
  $\BSReg = O(CK\sqrt{T})$.}
\end{corollary}

\end{document}